\documentclass{article}
\usepackage{arxiv}

\usepackage{microtype}
\usepackage{graphicx}
\usepackage{booktabs}

\usepackage{hyperref}
\hypersetup{
	colorlinks=true,
	linkcolor=magenta,
	citecolor=magenta,
}
\usepackage[natbib, citestyle=authoryear]{biblatex}
\usepackage{bm}
\usepackage{bbm}
\usepackage{amssymb}
\usepackage{amsmath}
\usepackage[nameinlink,capitalise,noabbrev]{cleveref}

\usepackage{amsthm}
\newtheorem{theorem}{Theorem}
\newtheorem{proposition}{Proposition}
\newtheorem*{theorem*}{Theorem}
\newtheorem*{proposition*}{Proposition}
\newtheorem{lemma}{Lemma}

\usepackage{pgfplots}
\pgfplotsset{compat=1.16}
\usepackage{subcaption}
\usepackage{xcolor}
\definecolor{light_gray}{HTML}{f0f0f0}
\definecolor{mid_gray}{HTML}{d9d9d9}
\definecolor{light_blue}{HTML}{306EFF}
\definecolor{light_orange}{HTML}{F87217}

\definecolor{neur_purpl}{HTML}{C099F0}
\definecolor{outt_purpl}{HTML}{5E4592}
\definecolor{neur_green}{HTML}{A8D8C2}
\definecolor{outt_green}{HTML}{278270}

\title{On the convergence of group-sparse autoencoders}
\date{}

\DeclareMathOperator*{\argmin}{\arg\min}

\author{Emmanouil Theodosis\\
	School of Engineering and Applied Sciences\\
	Harvard University\\
	Cambridge, MA 02138\\
	\texttt{etheodosis@gharvard.edu}\\
	\And
    Bahareh Tolooshams\\
    School of Engineering and Applied Sciences\\
    Harvard University\\
    Cambridge, MA 02138\\
    \And
    Pranay Tankala\\
    School of Engineering and Applied Sciences\\
    Harvard University\\
    Cambridge, MA 02138\\
    \And
    Abiy Tasissa\\
    Department of Mathematics\\
    Tufts University\\
    Medford, MA 02155\\
    \And
    Demba Ba\\
    School of Engineering and Applied Sciences\\
    Harvard University\\
    Cambridge, MA 02138\\
}

\begin{document}
\maketitle
\begin{abstract}
	Recent approaches in the theoretical analysis of model-based deep learning architectures have studied the convergence of gradient descent in shallow ReLU networks that arise from generative models whose hidden layers are sparse. Motivated by the success of architectures that impose structured forms of sparsity, we introduce and study a group-sparse autoencoder that accounts for a variety of generative models, and utilizes a  group-sparse ReLU activation function to force the non-zero units at a given layer to occur in blocks. For clustering models, inputs that result in the same group of active units belong to the same cluster. We proceed to analyze the gradient dynamics of a shallow instance of the proposed autoencoder, trained with data adhering to a group-sparse generative model. In this setting, we theoretically prove the convergence of the network parameters to a neighborhood of the generating matrix. We validate our model through numerical analysis and highlight the superior performance of networks with a group-sparse ReLU compared to networks that utilize traditional ReLUs, both in sparse coding and in parameter recovery tasks. We also provide real data experiments to corroborate the simulated results, and emphasize the clustering capabilities of structured sparsity models.
\end{abstract}

% keywords can be removed

\section{Introduction}
\label{sec:intro}

Model-based learning approaches \citep{BJP+17,SE19,SWE+20,TTT+20} address one of the fundamental problems of deep learning; that is, current high-performing architectures lack explainability. The model-based learning paradigm addresses this by invoking domain knowledge in order to constrain the neural architectures, thus making them amenable to interpretation. Within this context, \emph{unfolded networks} \citep{HRW14,TDB18,MLE20}, popularized by the seminal work of \citet{GL10}, unroll the steps of iterative optimization algorithms to form a neural network. Compared to their unconstrained counterparts, these structured architectures, which combine ideas from the signal processing and deep learning communities, enjoy a de facto reduction in the number of trainable parameters while maintaining competitive performance \citep{SE19, TST+20}.

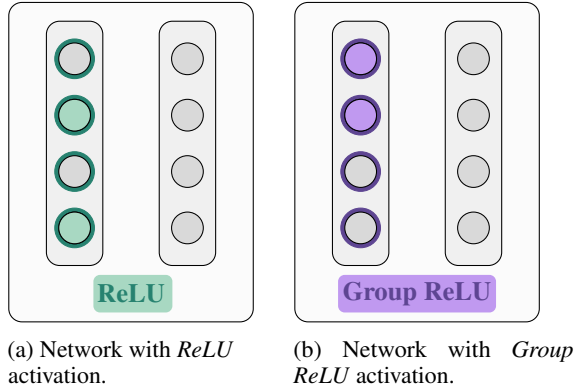
\begin{figure}[t]
    \centering
    \begin{subfigure}[t]{0.225\textwidth}
        \begin{tikzpicture}
            % Containing
            \draw[fill=light_gray, fill opacity=0.25, rounded corners=5pt] (-0.5, -0.75) rectangle (2.75, 3.5);

            % A^T
            \draw[fill=light_gray, rounded corners=5pt] (0, 0) rectangle (0.75, 3.25);
            \filldraw[outt_green] (0.375, 2.75) circle (7.5pt);
            \draw[fill=mid_gray] (0.375, 2.75) circle (6pt);

            \filldraw[outt_green] (0.375, 2) circle (7.5pt);
            \draw[fill=mid_gray] (0.375, 2) circle (6pt);
            \draw[fill=neur_green] (0.375, 2) circle (6pt);

            \filldraw[outt_green] (0.375, 1.25) circle (7.5pt);
            \draw[fill=mid_gray] (0.375, 1.25) circle (6pt);

            \filldraw[outt_green] (0.375, 0.5) circle (7.5pt);
            \draw[fill=mid_gray] (0.375, 0.5) circle (6pt);
            \draw[fill=neur_green] (0.375, 0.5) circle (6pt);

            % A
            \draw[fill=light_gray, rounded corners=5pt] (1.5, 0) rectangle (2.25, 3.25);

            % neurons
            \draw[fill=mid_gray] (1.875, 2.75) circle (6pt);
            \draw[fill=mid_gray] (1.875, 2) circle (6pt);
            \draw[fill=mid_gray] (1.875, 1.25) circle (6pt);
            \draw[fill=mid_gray] (1.875, 0.5) circle (6pt);

            % ReLU
            \draw[draw=none, fill=neur_green, rounded corners=2.5pt] (0.625, -0.625) rectangle (1.675, -0.125);
            \node[color=outt_green] at (1.125, -0.375) {\textbf{ReLU}};
        \end{tikzpicture}
        \caption{Network with \emph{ReLU}\\ activation.}
    \end{subfigure}
    \begin{subfigure}[t]{0.225\textwidth}
        \begin{tikzpicture}
             % Containing
            \draw[fill=light_gray, fill opacity=0.25, rounded corners=5pt] (-0.5, -0.75) rectangle (2.75, 3.5);

            % A^T
            \draw[fill=light_gray, rounded corners=5pt] (0, 0) rectangle (0.75, 3.25);

            % neurons
            \filldraw[outt_purpl] (0.375, 2.75) circle (7.5pt);
            \draw[fill=mid_gray] (0.375, 2.75) circle (6pt);
            \draw[fill=neur_purpl] (0.375, 2.75) circle (6pt);

            \filldraw[outt_purpl] (0.375, 2) circle (7.5pt);
            \draw[fill=mid_gray] (0.375, 2) circle (6pt);
            \draw[fill=neur_purpl] (0.375, 2) circle (6pt);

            \filldraw[outt_purpl] (0.375, 1.25) circle (7.5pt);
            \draw[fill=mid_gray] (0.375, 1.25) circle (6pt);

            \filldraw[outt_purpl] (0.375, 0.5) circle (7.5pt);
            \draw[fill=mid_gray] (0.375, 0.5) circle (6pt);

            % A
            \draw[fill=light_gray, rounded corners=5pt] (1.5, 0) rectangle (2.25, 3.25);

            % neurons
            \draw[fill=mid_gray] (1.875, 2.75) circle (6pt);
            \draw[fill=mid_gray] (1.875, 2) circle (6pt);
            \draw[fill=mid_gray] (1.875, 1.25) circle (6pt);
            \draw[fill=mid_gray] (1.875, 0.5) circle (6pt);

            % Group ReLU
            \draw[draw=none, fill=neur_purpl, rounded corners=2.5pt] (0.075, -0.625) rectangle (2.175, -0.125);
            \node[color=outt_purpl] at (1.125, -0.375) {\textbf{Group ReLU}};
        \end{tikzpicture}
        \caption{Network with \emph{Group ReLU} activation.}
    \end{subfigure}
	\caption{In sparsity-promoting models (\textbf{left}), arbitrary neurons can activate. In contrast, in group-sparse networks (\textbf{right}), the activated neurons are structured.}
	\label{fig:one}
\end{figure}

LISTA \citep{GL10}, which enforces sparsity on the units of deep layers, is based on the unfolding of the Iterative Shrinkage Thresholding Algorithm (ISTA), a sparse coding optimization algorithm. Sparsity-focused generative models \citep{T96} are most frequently employed due to their experimentally and theoretically proven generalization power \citep{MBP+09,MG13}. In addition, because they can significantly reduce the number of nonzero coefficients---units active at a given layer---sparse models have also been used to speed up inference in deep neural networks \citep{LWF+15,SCY+20}. Recent research has deviated from the traditional sparse coding model; certain works reconsidered the sparsity-promoting minimization \citep{HT18}, where others focused on exploring different generative models \citep{SCH+17,NUD17,AJ18}. Within the latter class, works studying \emph{group sparsity} \citep{YL06,EKB10} have been rather prolific. In addition to minimizing the number of non-zero coefficients, group sparsity forces them to occur in blocks (see \cref{fig:one}). As we argue in the sequel, we can interpret inputs that share active groups as belonging to the same class or cluster. The groupings manifest themselves either as a direct arrangement of the hidden units of neural networks into blocks \citep{WWW+16,YH17}, or as a clustering of data that, a priori, share similar characteristics, such as patches of natural images \citep{LPM20}.  Enforcing group structure has proved practical in applications and outperforms approaches based on the traditional notion of sparsity.

Despite the success of model-based unfolded architectures, their theoretical analysis is still nascent. The theoretical convergence of sparsity-based generative models has been studied in a \emph{supervised} setting, where data and their sparse representations are readily available \citep{CLW+18,LCW+19}. The majority of recent work, motivated by the general framework for the convergence of \emph{unsupervised} autoencoders introduced by \citet{AGM+15}, provide guarantees under which gradient descent will converge. For example, \citet{RMB+18} theoretically prove that the gradient of the loss function vanishes around critical points, and introduce a proxy gradient to approximate the true expectation. Most recently, \citet{NWH19} showed that a multitude of shallow models (derived from sparse coding) conform to the framework and converge to the generative model when trained with gradient descent. However, all of the existing literature is limited to models that rely on traditional notions of sparsity and, thus, does not account for generative models with structured notions of sparsity, such as ones with group-structured activations.

Our work poses the following question: given data generated according to a group-sparse model, \emph{can a shallow, group-sparse autoencoder recover the generating matrix, as well as nonzero blocks and their activations?} We answer this question in the affirmative, and our overall contributions are summarized as follows:

\textbf{Introduction of group-sparse autoencoders.} Motivated by the connection between clustering and group-sparsity, we introduce a \emph{group-sparse} autoencoder architecture, that we subsequently analyze theoretically and experimentally.

\textbf{Convergence of gradient descent.} Under mild assumptions introduced in \cref{sec:conve}, we prove that shallow group-sparse autoencoders, when trained with gradient descent, converge to a \emph{neighborhood} of the generative model.

\textbf{Recovering cluster membership.} In \cref{sec:conve} we show that the encoder's output demonstrably identifies nonzero blocks and their activations. Through the connection we establish between group-sparsity and clustering (\cref{sec:group}), this result suggests that, under mild conditions, we can recover cluster membership in unions of subspaces models. Finally, we showcase the clustering structure of group-sparsity in our experimental section.

\textbf{Group-sparse ReLU outperforms ReLU.} Through our experimental validation in \cref{sec:exper}, we demonstrate that a network with group-sparse activations shows superior performance compared to networks that promote traditional notions of sparsity. In particular, in both simulated and real data experiments, the group-sparse autoencoder outperforms a sparse autoencoder, both in recovering generative models and in learning interpretable dictionaries.

In what follows we denote matrices with capital bold-faced letters, vectors with lowercase bold-faced letters, and scalars with lowercase letters. $\bm{A}$ is a matrix, $\bm{A}_{S}$ is a sub-matrix of $\bm{A}$ indexed by $S$, $\bm{a}_i$ is the $i$-th column of $\bm{A}$, and $a_{ij}$ is the element in its $i$-th row and $j$-th column. $\lVert\bm{x}\rVert_p$ denotes the $\ell_p$ norm of $\bm{x}$ and $\lVert\bm{A}\rVert_2, \lVert\bm{A}\rVert_F$ denote the spectral and Frobenius norms, respectively, of $\bm{A}$, while $\sigma_1(\bm{A})$ denotes its maximum singular value. Finally, $[N] = \{1, \ldots, N\}$.

\section{Group-sparsity in dictionary learning}
\label{sec:group}
In model-based approaches the observed data $\{\bm{y}_i\}_{i=1}^{N} \in \mathcal{Y}$\footnote{We intentionally do not write $\{(\bm{x}_i, \bm{y}_i)\}_{i=1}^N$, as the setting we are considering is \emph{strictly} unsupervised.} are assumed to adhere to a generative model. Formally, we assume that the data satisfy
\begin{equation}
    \bm{y}_i = f_{\theta^\ast}(\bm{x}_i^{\ast}),
\end{equation}
where $\bm{x}^{\ast}_i \in \mathcal{X}$ is a latent vector and $f_{\theta}$, parametrized by $\theta$, comes from a function class $\mathcal{F}$ that describes the relation between the data $\bm{y}_i$ and the latent variables $\bm{x}^{\ast}_i$. Most frequent are models of \emph{linear} relations, where the function $f_{\theta}$ is of the form $f_\theta \colon \bm{x} \mapsto \bm{A}\bm{x}$, parametrized by $\theta = \{\bm{A}\}$.

\subsection{Group-sparse generative model}
Consider a generative model where each observation\footnote{For the rest of the text, we drop the index $i$ to reduce clutter.} $\bm{y}$ belongs to the union of one, or more, subspaces \citep{GN03}. In this general group-sparse model the observed data satisfy
\begin{equation}
    \label{eq:group_model}
    \bm{y} = \bm{A}^{\ast}\bm{x}^{\ast} = \sum_{g \in S}\bm{A}^{\ast}_g\bm{x}^{\ast}_g,
\end{equation}
where $S \subset [\Gamma]$ denotes the \emph{group support} (i.e.\ which of the $\Gamma$ groups are active), and the latent vector has the form $\bm{x}^{\ast} = [\bm{x}^{\ast}_1, \bm{x}^{\ast}_2, \ldots, \bm{x}^{\ast}_\Gamma]^T$. Gaussian mixture models, sparse models, and nonnegative sparse models \citep{NWH19} can readily be derived as special cases of the highly-expressive generative model from \eqref{eq:group_model}. The group-sparse prior assumes that the latent representation $\bm{x}$ is sparse, and that in nonzero entries occur in blocks (groups). %but also exhibits a grouping structure;
The model also implies a decomposition of $\bm{A}^{\ast}$ into sub-matrices $\bm{A}^{\ast}_1, \bm{A}^{\ast}_2, \ldots, \bm{A}^{\ast}_\Gamma$ such that $\bm{A}^{\ast} = [\bm{A}^{\ast}_1 \bm{A}^{\ast}_2 \ldots \bm{A}^{\ast}_\Gamma]$, where we assume that each group $\bm{A}^{\ast}_g$ has exactly $d$ elements. Without additional structure, the generative model may not yield a unique solution; for example, \citet{EKB10} impose orthonormality on $\bm{A}_g^{\ast}$ to ensure uniqueness.

An analogue to the \emph{coherence} of a dictionary in sparse models (defined as $\mu = \max_{i \neq j} \lvert\bm{a}^{\ast T}_i \bm{a}^{\ast}_j\rvert$; the inner-product with the largest magnitude in $\bm{A}^{\ast}$) is the \emph{block coherence} of $\bm{A}^{\ast}$
\begin{equation}
    \mu_B = \max_{g\neq h} \frac{1}{d} \lVert \bm{A}^{\ast T}_g \bm{A}^{\ast}_h\rVert_2.
\end{equation}
Intuitively, coherence metrics give a sense of how correlated the different columns, or groups, of $\bm{A}^\ast$ are and directly affect the ability to recover latent vectors. Assuming normalized groups, as we will in this work, it holds that $0 \leq \mu_B \leq \mu \leq 1$.

\subsection{Group-sparse dictionary learning}
Assuming a linear underlying generative model, dictionary learning \citep{AEB06,MBP12} sets out to learn a dictionary $\bm{A}$ such that every vector $\bm{y}$ in a data set adopts a sparse representation as a linear combination of the columns of $\bm{A}$ using a vector $\bm{x}$. In group-sparse settings, given the dictionary $\bm{A}$, \emph{group-sparse coding} lets us find $\bm{x}$ as the solution to the optimization problem
\begin{equation}
    \label{eq:l0opt}
    \min_{\bm{x} \in \mathcal{X}}\quad \lVert\bm{x}\rVert_{\textrm{$\ell_0$/$\ell_2$}}, \qquad \text{ s.t. } \bm{y} = \bm{A} \bm{x}.
\end{equation}
The $\ell_0$/$\ell_2$, expressed as the $\ell_0$ pseudo-norm of the vector of $\ell_2$ norms $[\lVert\bm{x}_1\rVert_2, \lVert\bm{x}_2\rVert_2, \ldots, \lVert\bm{x}_\Gamma\rVert_2]^T$, norm minimizes the number of active groups. The combinatorial nature of $\ell_0$ pseudo-norm makes this optimization intractable in practice. A popular approach utilizes the $\ell_1$ norm instead, as a tractable convex relaxation of the optimization of \eqref{eq:l0opt}, yielding
\begin{equation}
    \label{eq:l1opt}
    \min_{\bm{x} \in \mathcal{X}} \quad \lVert\bm{x}\rVert_{\textrm{$\ell_1$/$\ell_2$}}, \qquad \text{s.t. } \bm{y} = \bm{A} \bm{x},
\end{equation}
where $\lVert\bm{x}\rVert_{\textrm{$\ell_1$/$\ell_2$}} = \sum_{g \in S}\lVert\bm{x}_g\rVert_2$. Both the optimizations of \eqref{eq:l0opt} and \eqref{eq:l1opt} require the recovery of latent codes $\bm{x}$ that lead to an exact reconstruction of the data $\bm{y}$. The following \emph{unconstrained} optimization problem enables a trade-off between exact recovery and the group-sparsity of the latent codes
\begin{equation}
    \label{eq:unconst}
    \min_{\bm{x} \in \mathcal{X}}\quad \frac{1}{2}\lVert\bm{y} - \bm{A}\bm{x}\rVert_2^2 + \lambda \sum_{g \in S}\lVert\bm{x}_g\rVert_2.
\end{equation}
Optimization objectives of the form $\frac{1}{2}\lVert\bm{y} - \bm{A}\bm{x}\rVert_2^2 + \lambda \Omega(\bm{x})$ have been studied extensively in the literature and can be directly solved via the theory of proximal operators \citep{BJM+11,PB13}. Pertinent to the current discussion, the proximal operator promoting group-sparse structures can be derived as
\begin{equation}
    \label{eq:prox}
    \sigma_{\lambda}(\bm{x}_g) = \left(1 - \frac{\lambda}{\lVert\bm{x}_g\rVert_2}\right)_{+}\bm{x}_g,
\end{equation}
where $(\cdot)_{+} = \max(\cdot, 0)$. Note that this proximal operator bears a striking similarity to $\operatorname{ReLU}(x) = \max(x, 0)$. Indeed, we can consider \eqref{eq:prox} as a generalization of ReLU (informally termed ``Group ReLU''), where the thresholding is applied in a structured way, instead of an element-wise fashion. Dictionary learning can then be performed by solving
\begin{equation}
    \label{eq:dict}
    (\widehat{\bm{A}}, \widehat{\bm{x}}) = \argmin_{\bm{A} \in \mathcal{A}, \bm{x} \in \mathcal{X}} \frac{1}{2}\lVert\bm{y} - \bm{A}\bm{x}\rVert_2^2 + \lambda \sum_{g \in S}\lVert\bm{x}_g\rVert_2,
\end{equation}
a nonconvex optimization problem. A popular approach,  termed \emph{alternating minimization} \citep{T91, AAJ+16}, cycles between group-sparse coding and dictionary update steps.  Formally, the group-sparse coding step considers the dictionary $\widehat{\bm{A}}^k$ fixed and solves
\begin{equation}
    \label{eq:dict_gs}
    \widehat{\bm{x}}^{k+1} = \argmin_{\bm{x} \in \mathcal{X}} \quad \frac{1}{2} \lVert\bm{y} - \widehat{\bm{A}}^k\bm{x}\rVert_2^2 +  \lambda \sum_{g \in S}\lVert\bm{x}_g\rVert_2,
\end{equation}
followed by an optimization to find the optimal dictionary $\widehat{\bm{A}}^{k+1}$ given an estimate of the latent code $\widehat{\bm{x}}^{k+1}$
\begin{equation}
    \label{eq:dict_du}
    \widehat{\bm{A}}^{k+1} = \argmin_{\bm{A} \in \mathcal{A}} \quad \frac{1}{2} \lVert\bm{y} - \bm{A}\widehat{\bm{x}}^{k+1}\rVert_2^2,
\end{equation}
where \eqref{eq:dict_gs} and \eqref{eq:dict_du} are performed in an alternating manner until convergence, yielding $(\widehat{\bm{A}}^{\textrm{OPT}}, \widehat{\bm{x}}^{\textrm{OPT}})$.

\section{Group-sparse autoencoders and descent}
\label{sec:conve}
\subsection{Group-sparse autoencoder}
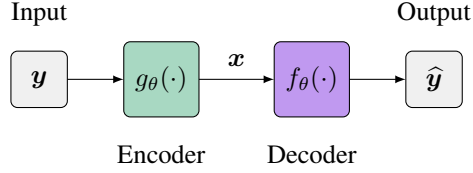
\begin{figure}
    \centering
     \begin{tikzpicture}
            % Input
            \draw[fill=light_gray, rounded corners=2.5pt] (-1.5, 0.125) rectangle (-0.75, 0.875);
            \node at (-1.125, 0.5) {$\bm{y}$};
            \node at (-1.125, 1.375) {Input};
            \draw[-latex] (-0.75, 0.5) -- (0, 0.5);

            % Encoder
            \draw[fill=neur_green, rounded corners=2.5pt] (0, 0) rectangle (1, 1);
            %\filldraw[neur_green, rounded corners=2pt] (0.015, 0.015) rectangle (0.985, 0.985);
            \node at (0.5, 0.5) {$g_{\theta}(\cdot)$};
            \node at (0.5, -0.5) {Encoder};
            \draw[-latex] (1, 0.5) -- (2, 0.5);

            % x_i
            \node at (1.5, 0.75) {$\bm{x}$};

            % Decoder
            \draw[fill=neur_purpl, rounded corners=2.5pt] (2, 0) rectangle (3, 1);
            %\filldraw[neur_purpl, rounded corners=2pt] (2.015, 0.015) rectangle (2.985, 0.985);
            \node at (2.5, 0.5) {$f_{\theta}(\cdot)$};
            \node at (2.5, -0.5) {Decoder};
            \draw[-latex] (3, 0.5) -- (3.75, 0.5);

            % Output
            \draw[fill=light_gray, rounded corners=2.5pt] (3.75, 0.125) rectangle (4.5, 0.875);
            \node at (4.125, 0.5) {$\widehat{\bm{y}}$};
            \node at (4.125, 1.375) {Output};
        \end{tikzpicture}
    \caption{Diagram of an autoencoder architecture.}
    \label{fig:two}
\end{figure}
The proximal operator of \eqref{eq:prox} implies an iterative optimization algorithm to recover the group-sparse codes $\bm{x}^{\ast}$. We unfold this optimization scheme into an autoencoder architecture \citep{GL10,TST+20}, illustrated in \cref{fig:two}, which implicitly recovers the sparse codes.
Formally, we denote the encoder as a function $g_{\theta}\colon \mathcal{R}^n \to \mathcal{R}^m$ such that
\begin{equation}
    g_{\theta}(\bm{y}) = \sigma_{\lambda}(\bm{A}^{T}\bm{y}) = \bm{x},
\end{equation}
and the corresponding decoder $f_{\theta}\colon \mathcal{R}^m \to \mathcal{R}^n$
\begin{equation}
    f_{\theta}(\bm{x}) = \bm{A}\bm{x} = \widehat{\bm{y}}.
\end{equation}
Notice that we deliberately use the same subscript $\theta$ (here $\theta = \{\bm{A}\}$), as the parametrization is common for the two systems, i.e. the weights are \emph{tied}.
We train the architecture by minimizing the loss function\footnote{Note that the loss function we are minimizing to train the autoencoder is \emph{different} from the minimization objective of \eqref{eq:unconst}, and instead similar to \eqref{eq:dict_du}.} $\mathcal{L}(\theta) = \tfrac{1}{2}\lVert\bm{y} - \widehat{\bm{y}}\rVert_2^2$. We assume that the data are being sampled from the generative model indicated by \eqref{eq:group_model} with $\bm{A} \in \mathcal{R}^{n\times m}$, which implies that $\bm{y} = \bm{A}^{\ast}\bm{x}^{\ast} = f_{\theta^\ast}(\bm{x}^{\ast})$.  Accounting for all training examples $\{\bm{y}_i\}_{i=1}^N$, we end up minimizing the cost
\begin{equation}
\frac{1}{N}\sum_{i=1}^N \mathcal{L}(\theta) = \tfrac{1}{2N}\sum_{i=1}^N\lVert\bm{y}_i - \widehat{\bm{y}}_i\rVert_2^2.
\end{equation}
The dictionary updates are performed by backpropagating through the architecture via gradient descent, yielding updates of the form
\begin{equation}
    \bm{A}^{k+1} = \bm{A}^{k} - \eta\nabla_{\bm{A}^{k}}\mathcal{L}(\bm{A}^{k}).
\end{equation}
In the next section, we will unfold a single iteration of the optimization and study a \emph{shallow} version of this network.

\subsection{Theoretical analysis}
Having introduced the shallow group-sparse autoencoder in the previous subsections, we will now analyze the theoretical properties of the architecture. We will organize our analysis in two parts; first, we will derive conditions that guarantee the recovery of the group support by applying the proximal operator of \eqref{eq:prox}. Then, assuming that the group memberships are correctly recovered, we will argue convergence of gradient descent by proving that the gradient of the loss function is aligned with $\bm{A}_g - \bm{A}_g^{\ast}$ .

\textbf{Assumptions.} We will delineate our critical assumptions here, for completeness; we will try our best to indicate where in the analysis each of the assumptions is invoked, as to also highlight why such a choice was made.\\
1. \emph{Bounded norms:} We assume that each group of the generating code $\bm{x}^\ast$ has bounded norm; i.e. for every $g$ in $S$ it holds that
\begin{equation}
    B_{\min} \leq \lVert\bm{x}_g^{\ast}\rVert_2 \leq B_{\max}.
\end{equation}
Note that this condition is significantly looser than similar conditions relating to sparse coding \citep{AGM+15,NWH19}. Precisely, as we will show experimentally in \cref{sec:exper}, the norm condition is more easily satisfied than a direct bound on the codes of the generative model.\\
2. \emph{Group-sparse support:} Each group $g$ has size $\operatorname{card}(g) = d$ and the number of non-zero groups is at most $\operatorname{card}(S) = \gamma$; the overall number of groups is $\Gamma$. Given that $\bm{A} \in \mathcal{R}^{n \times m}$, this implies that $m = d\cdot \Gamma$. We assume that the support $S$ is uniformly distributed, i.e.\ $\mathbb{P}[g \in S] = p_g = \Theta(\frac{\gamma}{\Gamma})$, $\mathbb{P}[g, h \in S] = p_{gh} = \Theta(\frac{\gamma^2}{\Gamma^2})$, etc.\\
3.\emph{Group-sparse code covariance:} Given two groups in the group support $v, h \in S$, we assume it holds
\begin{equation}
    \mathbb{E}\left[ \bm{x}^{\ast}_v\bm{x}^{\ast T}_h \mid v,h \in S \right] = \begin{cases}
			\bm{I}, & v = h,\\
			\bm{0}, & v \neq h.
		\end{cases}
\end{equation}
4. \emph{Model conditions:} We assume that our initialization, and therefore the subsequent updates, are relatively close to the generating matrix $\bm{A}^{\ast}$; formally, we assume that for every group $g \in S$ it holds that
\begin{equation}
    \lVert\bm{A}_g - \bm{A}^{\ast}_g\rVert_F \leq \delta,
\end{equation}
for some $\delta \in [0, 1]$. We also assume that the norm $\lVert\bm{A}_g^T\bm{A}_g^{\ast}\rVert_F$ is bounded, i.e.
\begin{equation}
    \lVert\bm{A}_g^T\bm{A}_g^{\ast}\rVert_F \leq \zeta,
\end{equation}
by some $\zeta$ that will be clarified later in this section. Finally, we assume that the weight matrix has orthonormal groups, i.e.\ $\bm{A}^{T}_g\bm{A}_g = \bm{I}$, and also $\lVert\bm{A}_g^{\ast}\rVert_2 = 1$. While that assumption might seem restrictive, it is fairly common in the literature of group sparsity \citep{EKB10}.

\textbf{Support recovery.} To show that the support is correctly recovered, we will show that for two groups $g, v \in \Gamma$ with $g \in S$ and $v \not \in S$ it holds that $\lVert\bm{A}^T_g\bm{y}\rVert_2 \geq \lambda$ and $\lVert\bm{A}^T_v\bm{y}\rVert_2 \leq \lambda$, for some threshold $\lambda$. Then the analysis consists of finding a suitable $\lambda$ guaranteeing that applying the proximal operator of \eqref{eq:prox} will attenuate the group $g$, but completely block group $v$. To that end, we decompose $\bm{A}_g^T\bm{y}$ as follows
\begin{equation}
    \bm{A}_g^T\bm{y} = \bm{A}_g^T\bm{A}^{\ast}_g\bm{x}^{\ast}_g + \sum_{h\neq g \in S}\bm{A}_g^T\bm{A}^{\ast}_h\bm{x}^{\ast}_h,
\end{equation}
that is, each group $g \in S$ comprises the contribution stemming from its own generating group and the cross-contribution of the other terms in the support $S$. Similarly, for a group $v \not \in S$, we will only have contributions from the ``cross-terms'', i.e.\ $\bm{A}_v^T\bm{y} = \sum_{h \in S}\bm{A}_v^T\bm{A}^{\ast}_h\bm{x}^{\ast}_h$. Therefore, to achieve support recovery we will show that the term $\bm{A}_g^T\bm{A}^{\ast}_g\bm{x}^{\ast}_g$ is large (in terms of norm) compared to the norm of $\sum_{h \in S}\bm{A}_v^T\bm{A}^{\ast}_h\bm{x}^{\ast}_h$ (with $v \not \in S$). The following two propositions bound the relevant norms.

\begin{proposition}[Group-norm lower bound]
\label{prop:lower}
The norm of the term $\bm{A}_g^T\bm{A}^{\ast}_g\bm{x}^{\ast}_g$ is lower-bounded by
\begin{equation}
    \lVert \bm{A}_g^T\bm{A}^{\ast}_g\bm{x}^{\ast}_g \rVert_2 \geq B_{\min}(1 - \delta).
\end{equation}
\end{proposition}

\begin{proposition}[Cross-term upper bound]
\label{prop:upper}
The norm of the term $\sum_{h \in S}\bm{A}_v^T\bm{A}^{\ast}_h\bm{x}^{\ast}_h$ is upper-bounded by
\begin{equation}
    \lVert \sum_{h \in S}\bm{A}_v^T\bm{A}^{\ast}_h\bm{x}^{\ast}_h \rVert_2 \leq \gamma B_{\max}(\mu_B + \delta).
\end{equation}
\end{proposition}

The proofs for \cref{prop:lower,prop:upper} can be found in Appendix A. Using these results, we can now lower-bound the norm of a group $g \in S$ as
\begin{equation}
    \lVert\bm{A}^T_g\bm{y}\rVert_2 \geq \lVert\bm{A}_g^T\bm{A}^{\ast}_g\bm{x}^{\ast}_g\rVert_2 -   \lVert \sum_{h \neq g \in S}\bm{A}_g^T\bm{A}^{\ast}_h\bm{x}^{\ast}_h \rVert_2 \geq B_{\min}(1-\delta) - \gamma B_{\max}(\mu_B + \delta).
\end{equation}
Similarly, we have $\lVert\bm{A}^T_v\bm{y}\rVert_2 \leq \gamma B_{\max}(\mu_B +\delta)$, for $v \not \in S$. If we choose the threshold $\lambda$ to be in the range $\left[\gamma B_{\max}(\mu_B + \delta), B_{\min}(1 - \delta) - \gamma B_{\max}(\mu_B + \delta) \right]$, then the support is correctly recovered. Connecting this result to the clustering narrative, cluster membership is recovered in a \emph{single} step of our architecture. Formally, we state the following lemma.
\begin{lemma}[Conditions for support recovery]
    \label{lem:recov}
    Assume $\gamma \leq \log n$ and $\max(\mu_B, \delta) = O(\frac{1}{\log n})$. Then the range $\left[\gamma B_{\max}(\mu_B + \delta), B_{\min}(1 - \delta) - \gamma B_{\max}(\mu_B + \delta) \right]$ is non-empty, and any value $\lambda$ within that range will correctly recover the support.
\end{lemma}

\textbf{Descent property.}
Having shown that the true groups are immediately recovered after a single application of the proximal operator, we will now proceed in arguing descent. In order to do that, we will first compute the gradient of the loss function with respect to each group $g \in S$, then, to deal with the fact that both the support and the group codes are unknown, take the expectation of the group (which invokes the \emph{infinite data} assumption). Having computed the gradient $\bm{G}_g$ of the loss function, we show that it is sufficiently aligned with $\bm{A}_g - \bm{A}^{\ast}_g$. This, in turn, is enough to guarantee that the error $\lVert\bm{A}^{k}_g - \bm{A}^{\ast}_g\rVert_F$ between the estimated group weights $\bm{A}^{k}_g$ at iteration $k$ and the generating matrix $\bm{A}^{\ast}_g$ is bounded, resulting to the convergence of the weights to a \emph{neighborhood} of $\bm{A}^{\ast}_g$.

The gradient of the loss function $\mathcal{L}(\theta)$ with respect to a group $\bm{A}_g$ for some $g \in \Gamma$ is given by\footnote{The complete derivation of the gradient can be found in Appendix B.}
\begin{equation}
        \nabla_{\bm{A}_g}\mathcal{L}(\theta) = -\left(\bm{y} - \bm{A}\sigma_{\lambda}(\bm{A}^T\bm{y})\right)\sigma_{\lambda}^T(\bm{A}^T_g\bm{y}) - \bm{y}\left(\bm{y} - \bm{A}\sigma_{\lambda}(\bm{A}^T\bm{y})\right)^T\bm{A}_g\operatorname{diag}(\sigma_{\lambda}'(\bm{A}^T_g\bm{y})).
\end{equation}
Considering the terms $\sigma_{\lambda}^T(\bm{A}^T_g\bm{y})$ and $\sigma_{\lambda}'(\bm{A}^T_g\bm{y})$, we can see that they will either be zero (when  $g\not\in S$), or they will be scaled versions of $\bm{A}^T_g\bm{y}$ and $\bm{1}$, respectively (where $\bm{1}$ denotes a vector of ones). Formally, we make the following substitutions
\begin{align}
    \begin{split}
    \label{eq:subs}
        \sigma_{\lambda}(\bm{A}_g^T \bm{y}) &= \mathbbm{1}_{\bm{x}_g \neq 0} \left(1 - \frac{\lambda}{\lVert\bm{A}_g^T\bm{y}\rVert_2}\right)\bm{A}_g^T\bm{y},\\
		\sigma'_{\lambda}(\bm{A}_g^T \bm{y}) &= \mathbbm{1}_{\bm{x}_g \neq 0} \left(1 - \frac{\lambda}{\lVert\bm{A}_g^T\bm{y}\rVert_2}\right) \bm{1}.
    \end{split}
\end{align}
The substitutions of \eqref{eq:subs} result in a different gradient $ \nabla_{\bm{A}_g}\widetilde{\mathcal{L}}(\theta)$ that was shown to be a good approximation of $ \nabla_{\bm{A}_g}\mathcal{L}(\theta)$ \citep{RMB+18}. We now need to deal with the last nonlinear term, $\sigma_{\lambda}(\bm{A}^T\bm{y})$. Since \eqref{eq:prox} attenuates the elements of each group in the support $S$ by a different term depending on the corresponding group norm, let us define $\tau_g = 	\left(1 - \frac{\lambda}{\lVert\bm{A}_{g}^T\bm{y}\rVert_2}\right)$ for every $g \in S$. Then, we can define a vector $\bm{\tau}$ such that
\begin{equation}
    \bm{\tau} = \begin{bmatrix}
		\smash{\underbrace{\begin{matrix}\tau_1 & \ldots & \tau_1\end{matrix}}_{d}} & \tau_2  & \ldots & \smash{\underbrace{\begin{matrix}\tau_{\gamma} & \ldots & \tau_{\gamma}\end{matrix}}_{d}}
		\end{bmatrix}^T.
		\vspace{1em}
\end{equation}
As suggested, the vector $\bm{\tau}$ is scaling every element of the group $g$ with the correct scaling of $\tau_g$. We can then write $\sigma_{\lambda}(\bm{A}^T\bm{y}) = \operatorname{diag}(\bm{\tau})\bm{A}^T_{S}\bm{y}$, and then the approximate gradient $ \nabla_{\bm{A}_g}\widetilde{\mathcal{L}}(\theta)$ becomes
\begin{equation}
    \label{eq:approx_grad}
    \nabla_{\bm{A}_g}\widetilde{\mathcal{L}}(\theta) = - \mathbbm{1}_{\bm{x}_g \neq 0}\cdot\tau_g \Big[\left(\bm{I} - \bm{A}_{S}\operatorname{diag}(\boldsymbol{\tau})\bm{A}_{S}^T\right)\bm{y}\bm{y}^T \bm{y}\bm{y}^T\left(\bm{I} - \bm{A}_{S}\operatorname{diag}(\boldsymbol{\tau})\bm{A}_{S}^T\right)^T\Big]\bm{A}_g.
\end{equation}
At this point, we will take the expectation of \eqref{eq:approx_grad}. The reasoning for this is that the true codes $\bm{x}_g^{\ast}$ (and their supports) are unknown, and thus we have the expected gradient $\bm{G}_g$
\begin{align}
    \begin{split}
        \bm{G}_g &= -\mathbb{E}\Big[ \mathbbm{1}_{\bm{x}^{\ast}_g \neq 0}\cdot\tau_g \left(\bm{I} - \bm{A}_{S}\operatorname{diag}(\boldsymbol{\tau})\bm{A}_{S}^T\right)\bm{y}\bm{y}^T\bm{A}_g\Big] -\mathbb{E}\Big[ \mathbbm{1}_{\bm{x}^{\ast}_g \neq 0}\cdot\tau_g \bm{y}\bm{y}^T\left(\bm{I} - \bm{A}_{S}\operatorname{diag}(\boldsymbol{\tau})\bm{A}_{S}^T\right)^T\bm{A}_g\Big] + \epsilon\\
        &= \bm{G}_g^{(1)} + \bm{G}_g^{(2)} + \epsilon,
    \end{split}
\end{align}
where $\mathbbm{1}_{\bm{x}^{\ast}_g}$ replaced $\mathbbm{1}_{\bm{x}_g}$, as we showed that the support is recovered with a single application of the proximal operator. Regardless, this introduces an error term $\epsilon$, that is, however, bounded (and specifically has a norm of order $O(n^{-w(1)})$ \citep{NWH19}). In order to deal with the unknown support, we will use Adam's Law and compute $\bm{G}^{(i)}_g$ as $\mathbb{E}\left[\bm{G}^{(i)}_{g\mid S}\right]$, where
\begin{equation}
    \bm{G}^{(i)}_{g\mid S} = - \mathbb{E}\left[\mathbbm{1}_{\bm{x}_g^\ast \neq 0} \cdot \tau_g \left(\bm{I} - \bm{A}_{S}\operatorname{diag}(\bm{\tau})\bm{A}_{S}^T\right)\bm{y}\bm{y}^T\bm{A}_{g} \mid S\right].
\end{equation}
Then, noting that $\bm{y}\bm{y}^T = \sum_{h,v \in S}\bm{A}_h^{\ast}\bm{x}_h^\ast \bm{x}_v^{\ast T} \bm{A}_v^{\ast T}$ and invoking the assumption of the code covariances we can finally write $\bm{G}^{(i)}_{g\mid S}$ as
\begin{align}
    \begin{split}
        \bm{G}^{(1)}_{g\mid \Gamma} &= -\tau_g \left(\bm{I} - \bm{A}_{S}\operatorname{diag}(\bm{\tau})\bm{A}^T_{S}\right)\bm{A}^{\ast}_g\bm{A}^{\ast T}_g\bm{A}_g + \bm{P}_1,\\
        \bm{G}^{(2)}_{g\mid S} &= -\tau_g \bm{A}^{\ast}_g\bm{A}^{\ast T}_g\left(\bm{I} - \bm{A}_{S}\operatorname{diag}(\bm{\tau})\bm{A}^T_{S}\right)\bm{A}_g + \bm{P}_2,\\
    \end{split}
\end{align}
where the $\bm{P}_i$ are cross-terms given by
\begin{align}
    \begin{split}
        \bm{P}_1 &= -\tau_g \left(\bm{I} - \bm{A}_{S}\operatorname{diag}(\bm{\tau})\bm{A}_{S}^T\right)\sum_{h\neq g \in S}\bm{A}_h^{\ast}\bm{A}_h^{\ast T}\bm{A}_{g},\\
        \bm{P}_2 &= -\tau_g \sum_{h\neq g \in S}\bm{A}_h^{\ast}\bm{A}_h^{\ast T}\left(\bm{I} - \bm{A}_{S}\operatorname{diag}(\bm{\tau})\bm{A}_{S}^T\right)\bm{A}_{g}.
    \end{split}
\end{align}
Then, noting that $\bm{A}_{S}\operatorname{diag}(\bm{\tau})\bm{A}^T_{S} = \sum_{h \in S}\tau_h\bm{A}_h\bm{A}^T_h$ and skipping intermediate steps we can finally write the expected gradient $\bm{G}_g$ as
\begin{equation}
    \label{eq:expected_gradient}
    \bm{G}_g = \tau_g p_g (2 - \tau_g)(\bm{A}_g - \bm{A}^{\ast}_g)\bm{A}^{\ast T}_g\bm{A}_g + \bm{V} + \epsilon,
\end{equation}
where the matrix $\bm{V}$ is an amalgam of $\mathbb{E}[\bm{P}_1 + \bm{P}_2]$ and lower probability terms. We can then prove the following theorem.
\begin{theorem}[Gradient direction]
    \label{thrm:aligned}
    Consider the expected gradient of \eqref{eq:expected_gradient}. The \emph{inner product} between the $i$-th columns of $\bm{G}_g$ and $\bm{A}_g - \bm{A}^{\ast}_g$ is lower-bounded by
    \begin{align}
        \label{eq:grad_align}
        \begin{split}
        2\langle \bm{g}_{gi}, \bm{a}_i - \bm{a}^{\ast}_i \rangle &\geq \tau_g(2-\tau_g)p_g\alpha_{i}\lVert\bm{a}_i - \bm{a}_i^{\ast}\rVert_2^2 + \frac{1}{\tau_g (2 - \tau_g) p_g \alpha_i}\lVert \bm{g}_{gi}\rVert_2^2 -\frac{1}{\tau_g (2 - \tau_g) p_g \alpha_i}\lVert \bm{v}_i\rVert_2^2\\
		&-O((\mu_B + \delta)^2\tfrac{\gamma^5}{\Gamma^3}),
        \end{split}
    \end{align}
    where $\alpha_i = \bm{a}^{\ast T}_i\bm{a}_i$ and $\lVert\bm{v}_i\rVert_2$ satisfies
    \begin{equation}
        \lVert\bm{v}_i\rVert_2 \leq \tau_g(2-\tau_g)p_g(\omega_{i}\sqrt{d^2 + 1} + \delta)\lVert\bm{A}_g - \bm{A}^{\ast}_g\rVert_F
    \end{equation}
    with $\omega{i} = \max_{j\neq i \in g} \lvert\bm{a}^{\ast T}_j\bm{a}_i\rvert$.
\end{theorem}

Intuitively, \cref{thrm:aligned} states that the gradient $\bm{g}_{gi}$ ``points'' at the same direction as $\bm{a}_i - \bm{a}^{\ast}_i$, and thus moving along the opposite direction will get us closer to the generative model of $\bm{A}_g$. We are finally able to state our main result.\footnote{The proof of both \cref{thrm:aligned,thrm:main_res} are given in Appendix C.}
\begin{theorem}[Convergence to a neighborhood]
    \label{thrm:main_res}
    Suppose that the learning rate $\eta$ is upper bounded by $\frac{1}{\tau_g(2-\tau_g)p_g\alpha_{\max}}$, the norm $\lVert\bm{A}_g^T\bm{A}_g^{\ast}\rVert_F^2$ is upper bounded by $\frac{1}{\tau_g^2}[3(2-\tau_g)(\tau_g -\frac{2}{3})n + (1-\tau_g)(2-\tau_g)\delta^2]$, and that $\bm{G}_g$ is ``aligned'' with $\bm{A}_g-\bm{A}_g^{\ast}$, i.e.
    \begin{equation*}
        2\langle \bm{g}_{gi}, \bm{a}_i - \bm{a}^{\ast}_i \rangle \geq \kappa \lVert\bm{a}_i - \bm{a}^{\ast}_i\rVert_2^2 + \nu \lVert \bm{g}_i\rVert_2^2 - \xi \lVert \bm{v}_i\rVert_2^2 - \varepsilon,
    \end{equation*}
    where $\kappa, \nu, \xi,$ and $\varepsilon$ are given by \cref{thrm:aligned}. Then it follows that
    \begin{equation*}
        \lVert\bm{A}^{k+1}_g - \bm{A}^{\ast}_g\rVert_F^2 \leq (1-\rho)\lVert\bm{A}^{k}_g - \bm{A}^{\ast}_g\rVert_F^2 + \eta d\varepsilon,
    \end{equation*}
    where $\rho = \eta\tau_g(2-\tau_g)p_g(\alpha_{\min}-\tfrac{2d(\omega_{\max}\sqrt{d^2 + 1} +\delta)^2}{\alpha_{\min}})$. If we further consider the assumptions of \cref{lem:recov} then
    \begin{equation*}
        \lVert\bm{A}^{k+1}_g - \bm{A}^{\ast}_g\rVert_F^2 \leq (1-\rho)\lVert\bm{A}^{k}_g - \bm{A}^{\ast}_g\rVert_F^2 + O(d\tfrac{\log^2n}{\Gamma^2}),
    \end{equation*}
	where $\omega_{\max} = \max_{i} \omega_i$, $\alpha_{\min} = \min_i \alpha_i$, and $\alpha_{\max} = \max_i \alpha_i$.
\end{theorem}
This concludes the proof of descent; we showed that, when training with gradient descent, the weights of a shallow group-sparse autoencoder will, in Frobenius norm, get closer to the weights of the generating model. However, because of the error term $\varepsilon$, the convergence isn't to the exact weights; it is rather to a \emph{neighborhood} of the generative model.

\textbf{Comparison with prior work.} Existing literature has analyzed the convergence of gradient descent assuming a sparse generative model \citep{AGM+15,RMB+18,NWH19}. As our study assumes a group-sparse model, it is natural to ask how these assumptions and analyses differ. We identify two main differences in favor of the group-sparse approach:\\
\textbf{(a)} Works using traditional notions of sparsity \citep{AGM+15,RMB+18,NWH19} require bounds on the true codes of the generative model. Instead, our formulation requires a bound on the \emph{norms} of the true codes; a significantly looser assumption. As a result, a group-sparse autoencoder is able to recover the support of the true codes $\bm{x}^{\ast}$ with a significantly higher success rate (\cref{sec:exper}).\\
\textbf{(b)} Sparse approaches \citep{NWH19} require the magnitude of the bias to \emph{diminish} at every iteration. In stark contrast, our analysis makes no such assumption; we are able to satisfy both support recovery and convergence using a \emph{constant} choice of bias.

\begin{figure*}[ht]
    \centering
	\begin{subfigure}[t]{0.24\textwidth}
	    \centering
	    \includegraphics[width=\textwidth]{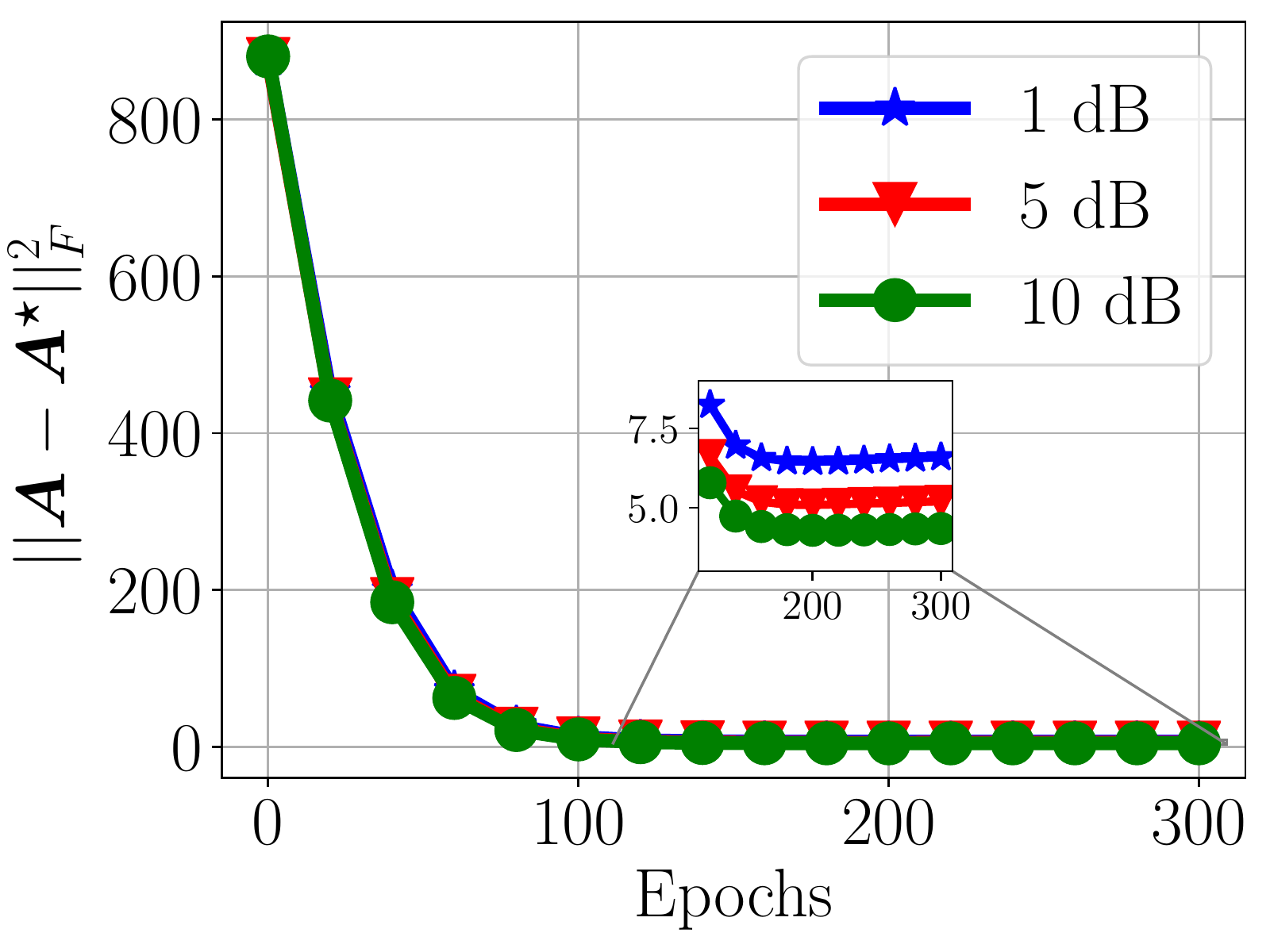}
	    \caption{Recovery of the generating matrix (in terms of the Frobenius norm) for various noise levels.}
	    \label{fig:threea}
	\end{subfigure}\hfill
	\begin{subfigure}[t]{0.24\textwidth}
	    \centering
	    \includegraphics[width=\textwidth]{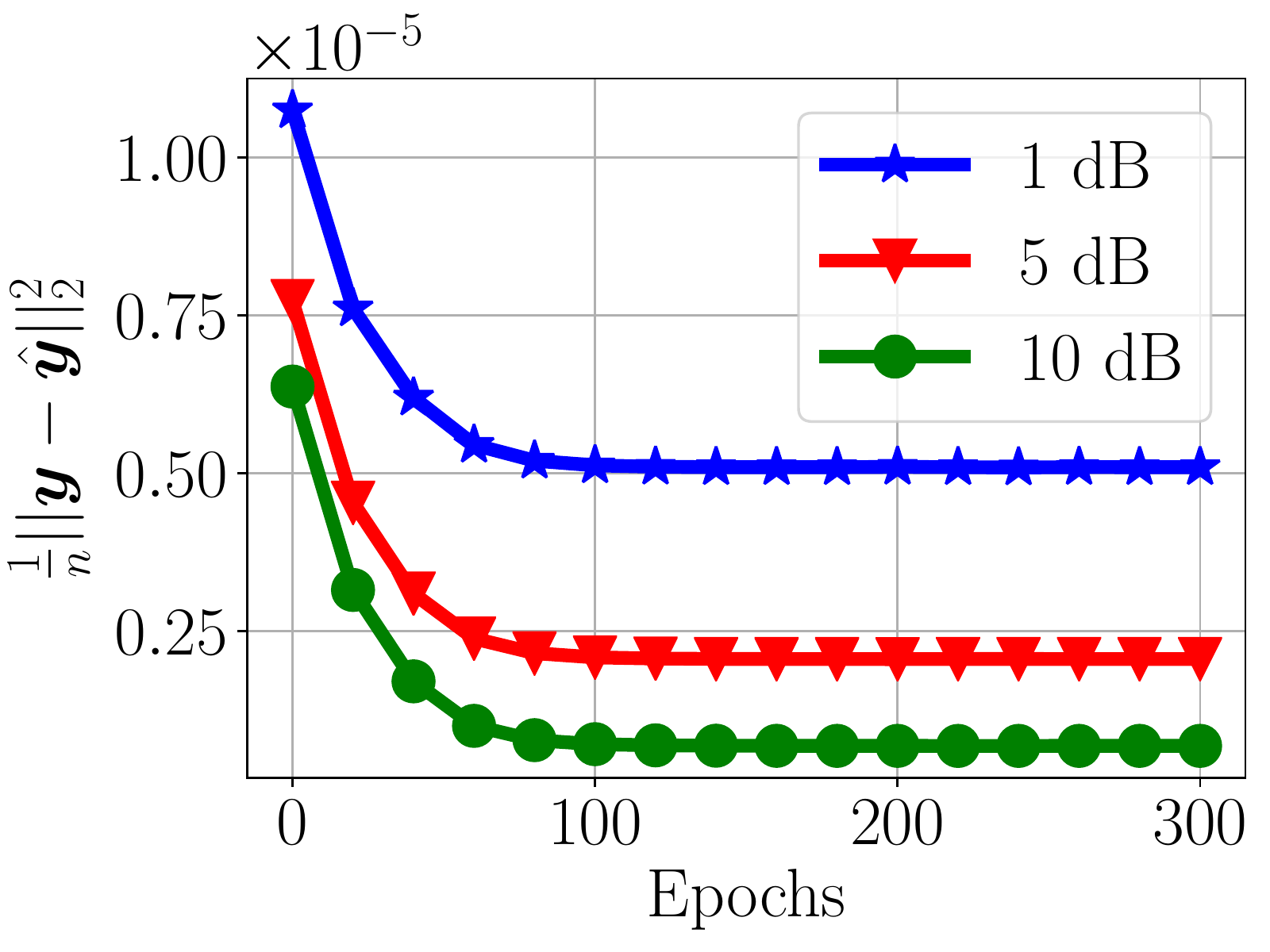}
	    \caption{Training loss of the group-sparse autoencoder across epochs for different noise levels.}
	    \label{fig:threeb}
	\end{subfigure}\hfill
	\begin{subfigure}[t]{0.24\textwidth}
	    \centering
	    \includegraphics[width=\textwidth]{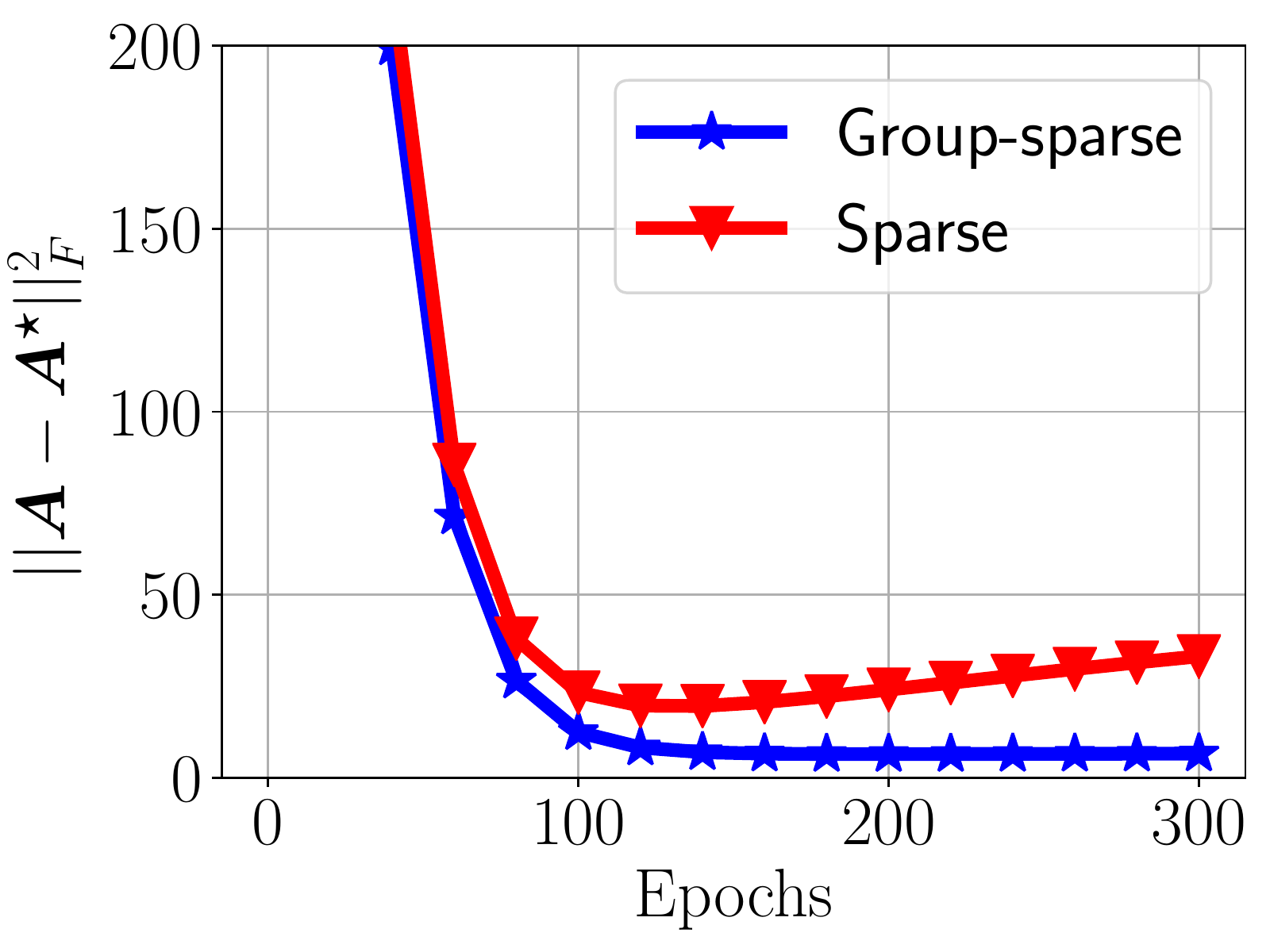}
	    \caption{Comparing a group-sparse and a sparse network in the recovery of the generating matrix.}
	    \label{fig:threec}
	\end{subfigure}\hfill
	\begin{subfigure}[t]{0.24\textwidth}
        \centering
	    \includegraphics[width=\textwidth]{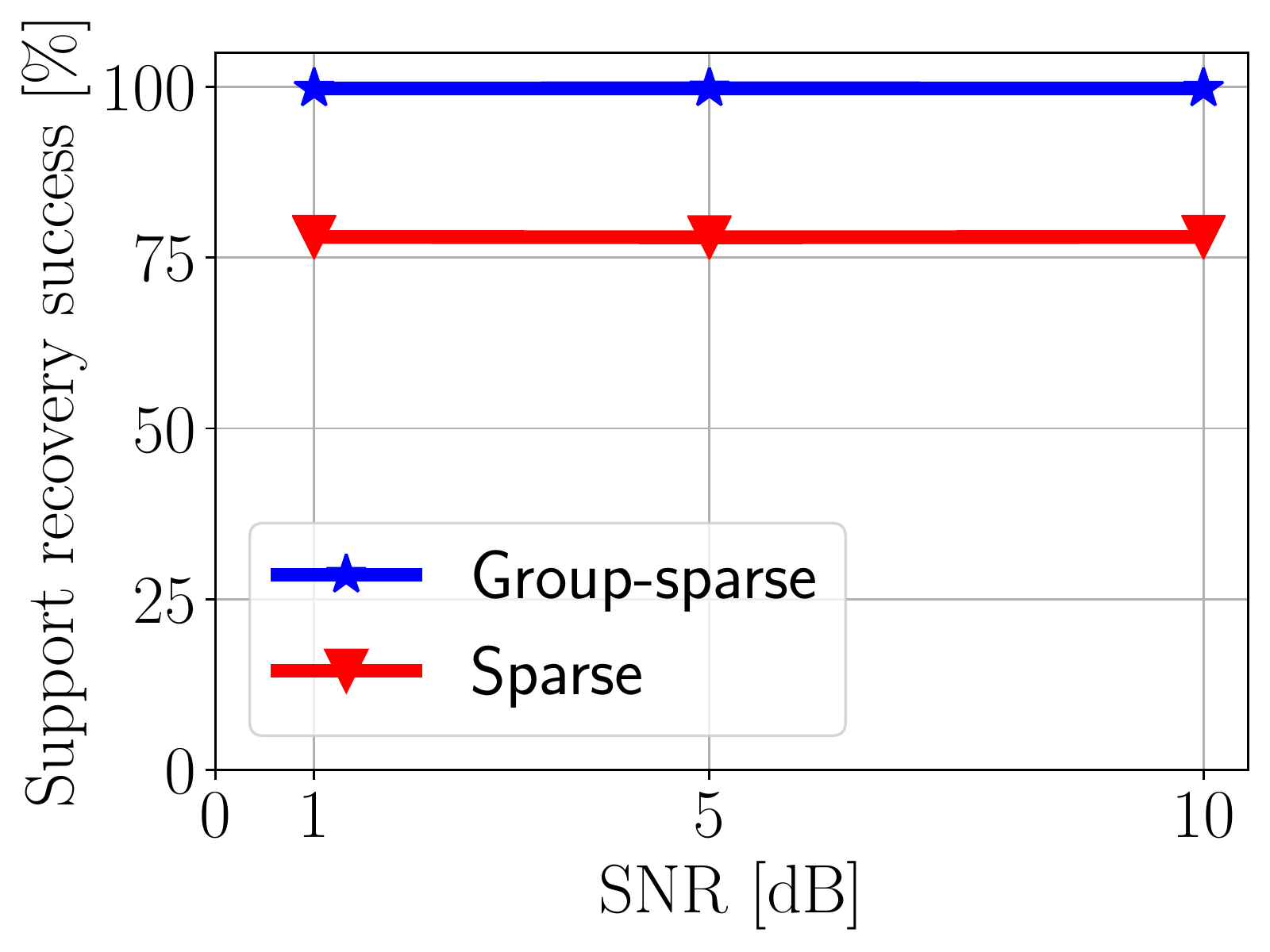}
	    \caption{Success rate in recovering the correct support for the group-sparse and sparse architectures.}
	    \label{fig:threed}
	\end{subfigure}
	\caption{Numerical simulations highlighting the superior performance of the proposed model versus a sparse autoencoder.}
\end{figure*}

\section{Experimental validation}
\label{sec:exper}
We support our theoretical results via numerical simulations (\cref{sec:sim_exp}) and evaluate the performance of group-sparse learning on a clustering task (\cref{sec:real_exp}). We highlight the superior performance of group-sparse coding compared to traditional sparse coding, and the natural clustering induced by the group-sparse autoencoder in real-data experiments.

\subsection{Simulation experiments}\label{sec:sim_exp}
In this section we show that training our proposed architecture recovers the parameters of the generative model and showcase superior performance to a traditional sparse autoencoder, both in recovering the generating dictionary, and in the ability of the encoder to identify nonzero blocks of units and their activations. % and sparse coding.

\textbf{Data generation.} We generated a set of data points $\{\bm{y}_i\}_{i=1}^N \in \mathcal{R}^{950}$ consisting of $N = 10{,}000$ examples following the noisy group-sparse generative model $\bm{y}_i = \bm{A}^{\ast}\bm{x}^{\ast}_i + \bm{z}_i$, where the data are corrupted by strong additive zero-mean Gaussian noise with signal-to-noise-ratio (SNR) equal to $1$, $5$, and $10$ dB. We let $\bm{A}^{\ast} \in \mathcal{R}^{n \times m}$ to comprise $\Gamma = 500$ groups, each of size $d = 2$ (i.e. $m = 1000$). We sampled the entries of each column $\bm{a}^{\ast}_i$ from a zero-mean Gaussian distribution and normalized each column.  The codes $\bm{x}^{\ast}_i \in \mathbb{R}^{m}$ contains $3$ active groups chosen uniformly at random from the set $[\Gamma]$ (i.e. the code is $6$-sparse). We sampled the code entries of active groups according to the $\mathcal{N}(0,1)$ distribution; after normalizing the vector of coefficients $\bm{x}_g$ in an active group, we scaled them with a factor $c \sim \text{Uniform}(4,5)$.

\textbf{Training.} We initialized the weights of a group-sparse autoencoder with an extreme perturbation of the generating dictionary (i.e. $\bm{A} = \bm{A^{\ast}} + \bm{B}$ with $\bm{B} \sim \mathcal{N}(\bm{0}, \sigma^2_B \bm{I})$), where the average correlation between the columns of $\bm{A}^{\ast}$ and $\bm{A}$ is approximately $0.15$. We chose this initialization procedure over a random initialization to make sure the weights are far from $\bm{A^{\ast}}$, but also to minimize the possibility of column permutations (so that we can compare $\bm{A}$ to $\bm{A}^{\ast}$ without solving the combinatorial problem of matching columns) \citep{AAJ+16, TDB20}. We trained the architecture for $300$ epochs with full-batch gradient descent using the Adam optimizer with a learning rate $\eta = 10^{-3}$ and bias $\lambda = 2$. During training, we did not constrain the weights to have unit norm. We measured the distance between the network weights and the generating ones at each iteration using the Frobenius norm $\lVert\bm{A} - \bm{A}^{\ast} \rVert_F^2$ of the difference between the normalized weights.

\textbf{Results.} \cref{fig:threea} supports our theory by demonstrating the convergence of $\bm{A} \rightarrow \bm{A}^{\ast}$ and how the error $\lVert \bm{A} - \bm{A}^{\ast} \rVert_F^2$ decreases as a function of epochs for various SNRs. This figure extends our theory and highlights the robustness of the network to extreme noise. \cref{fig:threeb} shows the reconstruction loss as a function of epochs for various SNRs.

We compared the group-sparse autoencoder with a sparse autoencoder for all SNRs. For a SNR of $1$ dB, \cref{fig:threec} shows that by taking into account the structured sparsity of the data, we estimate the generating dictionary better; the group-sparse network convergences to a closer neighbourhood of the dictionary than the sparse one. We attribute this superiority to the better sparse coding performance (i.e. better recovery of the support at the encoder) of the group-sparse autoencoder compared to that of the sparse network. Indeed, based on our theory, to ensure the correct recovery of the support, the group-sparse analysis only requires a bound on the \emph{norm} of each group. In contrast, the sparse autoencoder of \citet{NWH19}, enforces conditions on the energy of individual entries of the code. \cref{fig:threed} examines the codes at the output of the encoder; we observe that unlike the sparse network, the group-sparse autoencoder has successfully recovered the correct support.

\begin{figure*}[ht!]
    \centering
	\begin{subfigure}[t]{0.24\textwidth}
        \includegraphics[width=\textwidth]{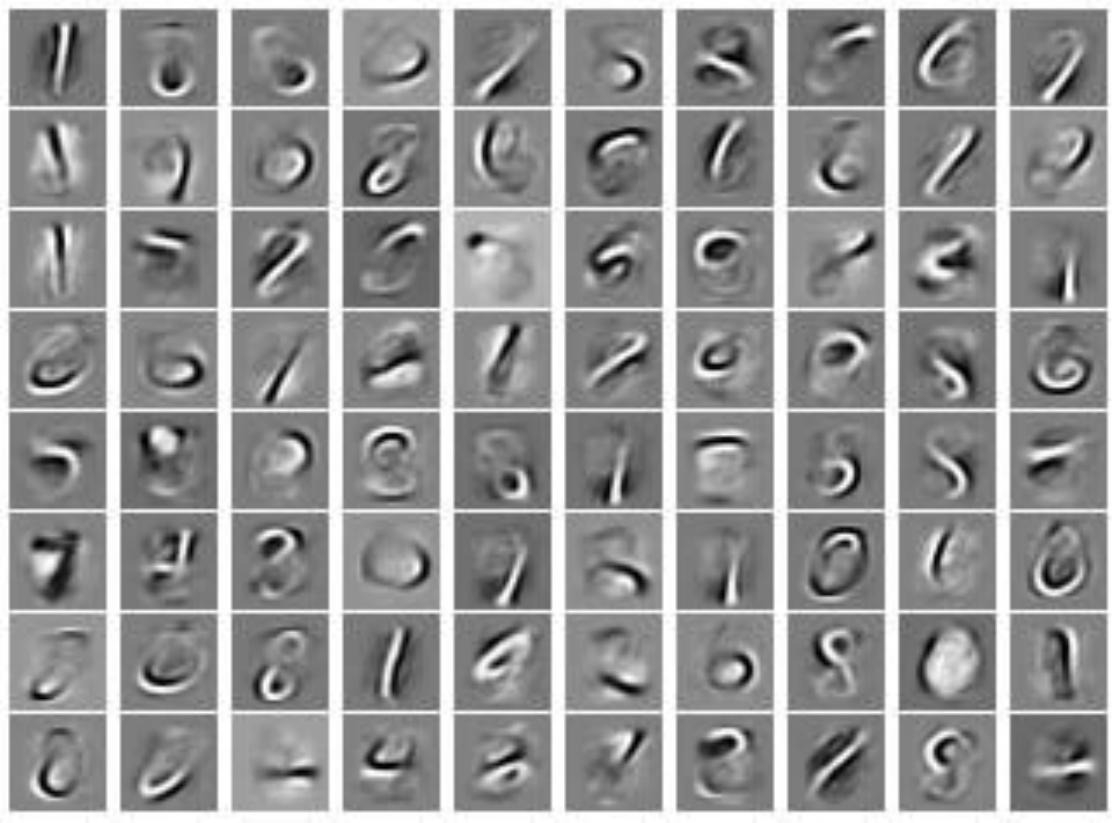}
	    \caption{Learned representations using a sparse autoencoder.}
	    \label{fig:mnist_dictionariesa}
	\end{subfigure}\hfill
	\begin{subfigure}[t]{0.24\textwidth}
        \includegraphics[width=\textwidth]{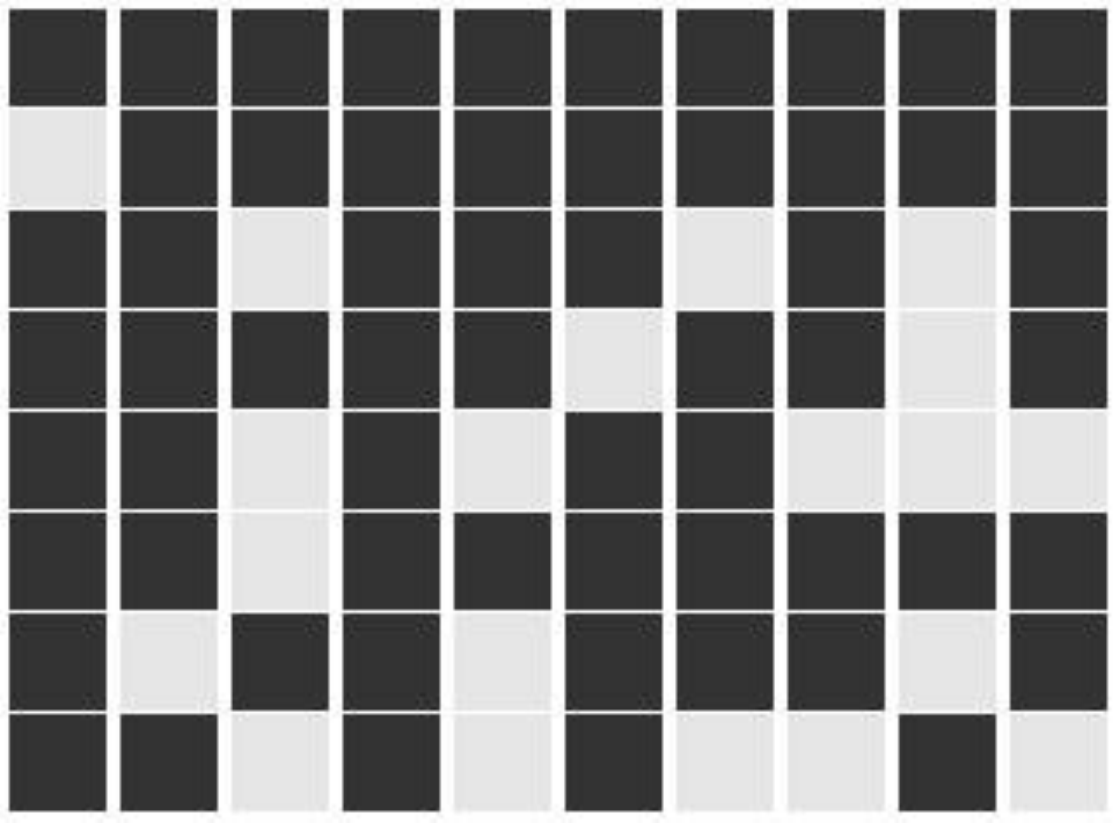}
	    \caption{Activation pattern of a code for the sparse architecture.}
	    \label{fig:mnist_dictionariesb}
	\end{subfigure}\hfill
	\begin{subfigure}[t]{0.24\textwidth}
        \includegraphics[width=\textwidth]{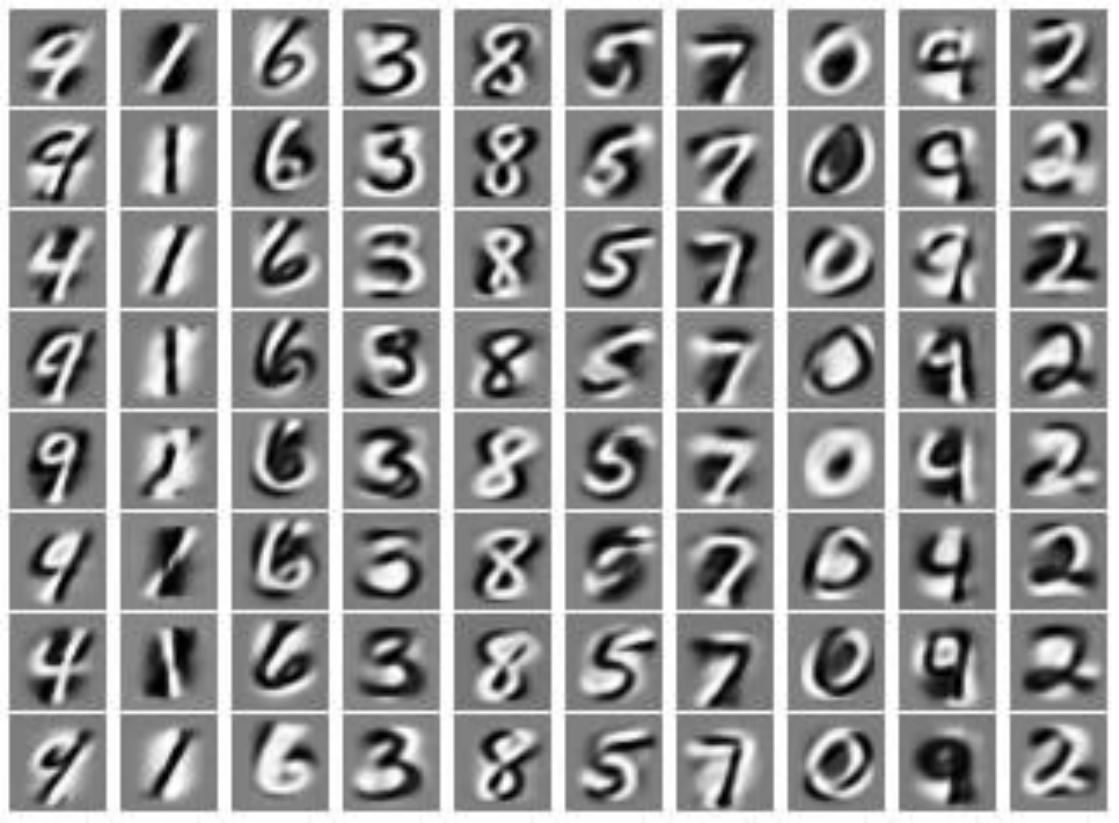}
	    \caption{Learned representations using a group-sparse autoencoder.}
	    \label{fig:mnist_dictionariesc}
	\end{subfigure}\hfill
	\begin{subfigure}[t]{0.24\textwidth}
        \includegraphics[width=\textwidth]{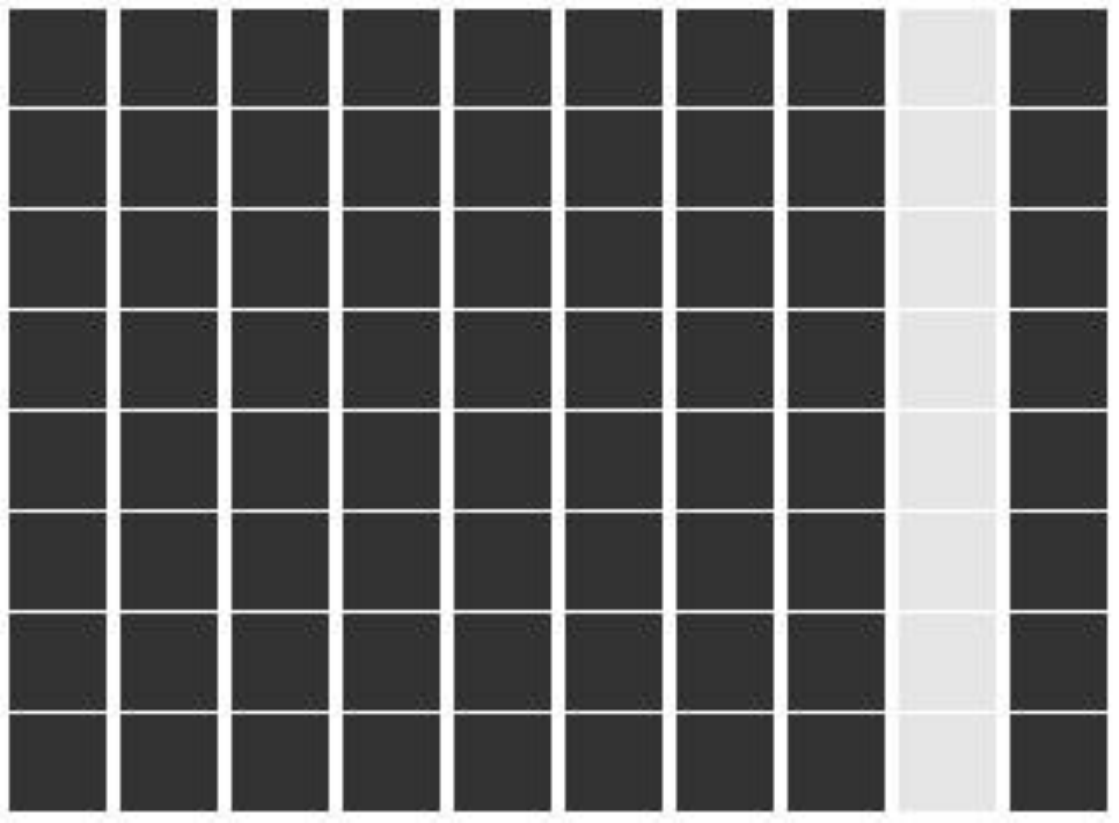}
	    \caption{Activation pattern of a code for the group-sparse architecture.}
	    \label{fig:mnist_dictionariesd}
	\end{subfigure}
	\caption{Representations and code activation patterns learned by the group-sparse and sparse autoencoders.}
	\label{fig:mnist_dictionaries}
\end{figure*}

\begin{figure*}[ht!]
    \centering
	\begin{subfigure}[t]{0.32\textwidth}
	    \centering
        \includegraphics[width=0.9\textwidth]{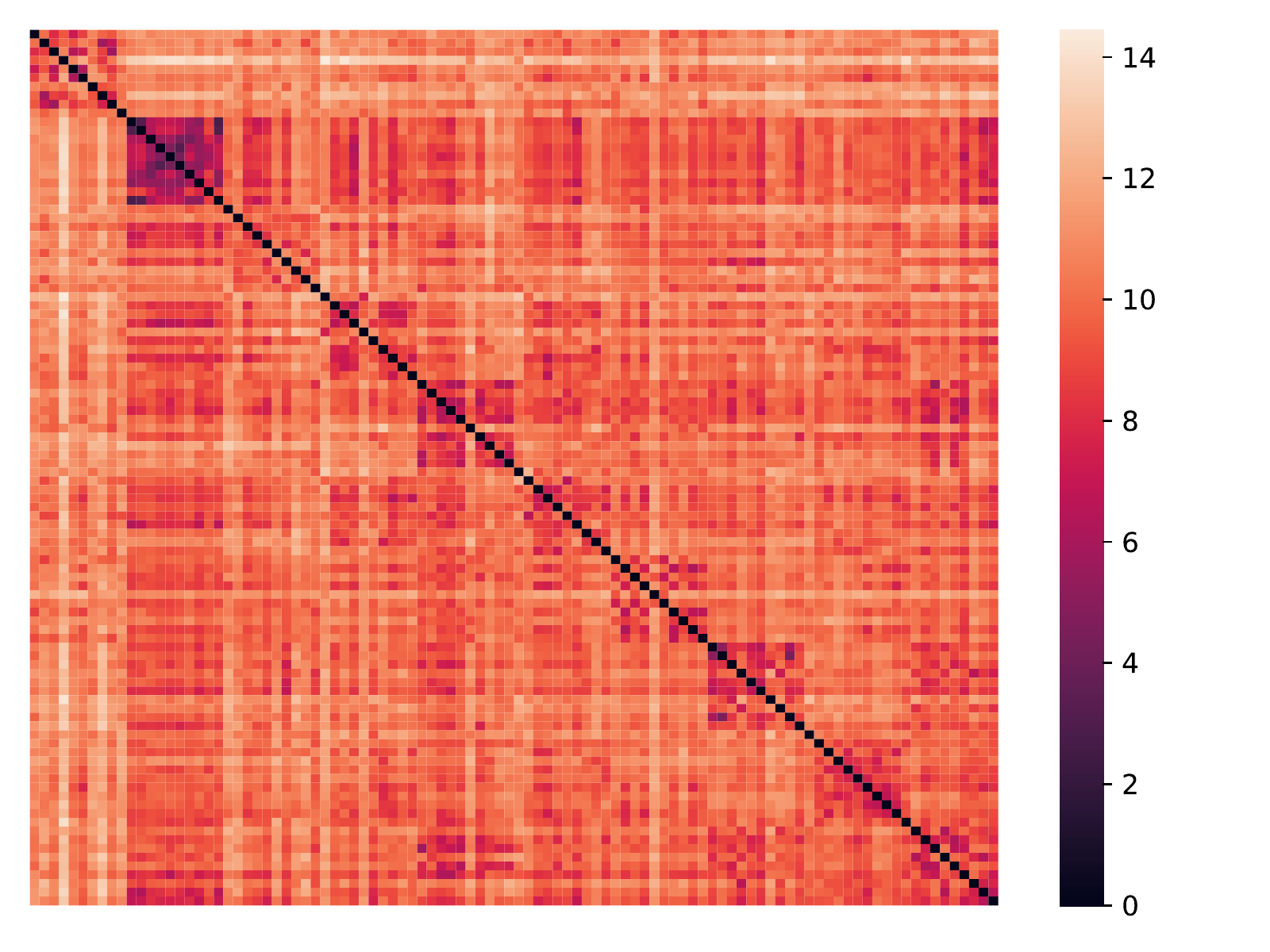}
	    \caption{Similarity matrix of MNIST images when representing the data using their raw format (pixel basis).}
	\end{subfigure}\hfill
	\begin{subfigure}[t]{0.32\textwidth}
	    \centering
        \includegraphics[width=0.9\textwidth]{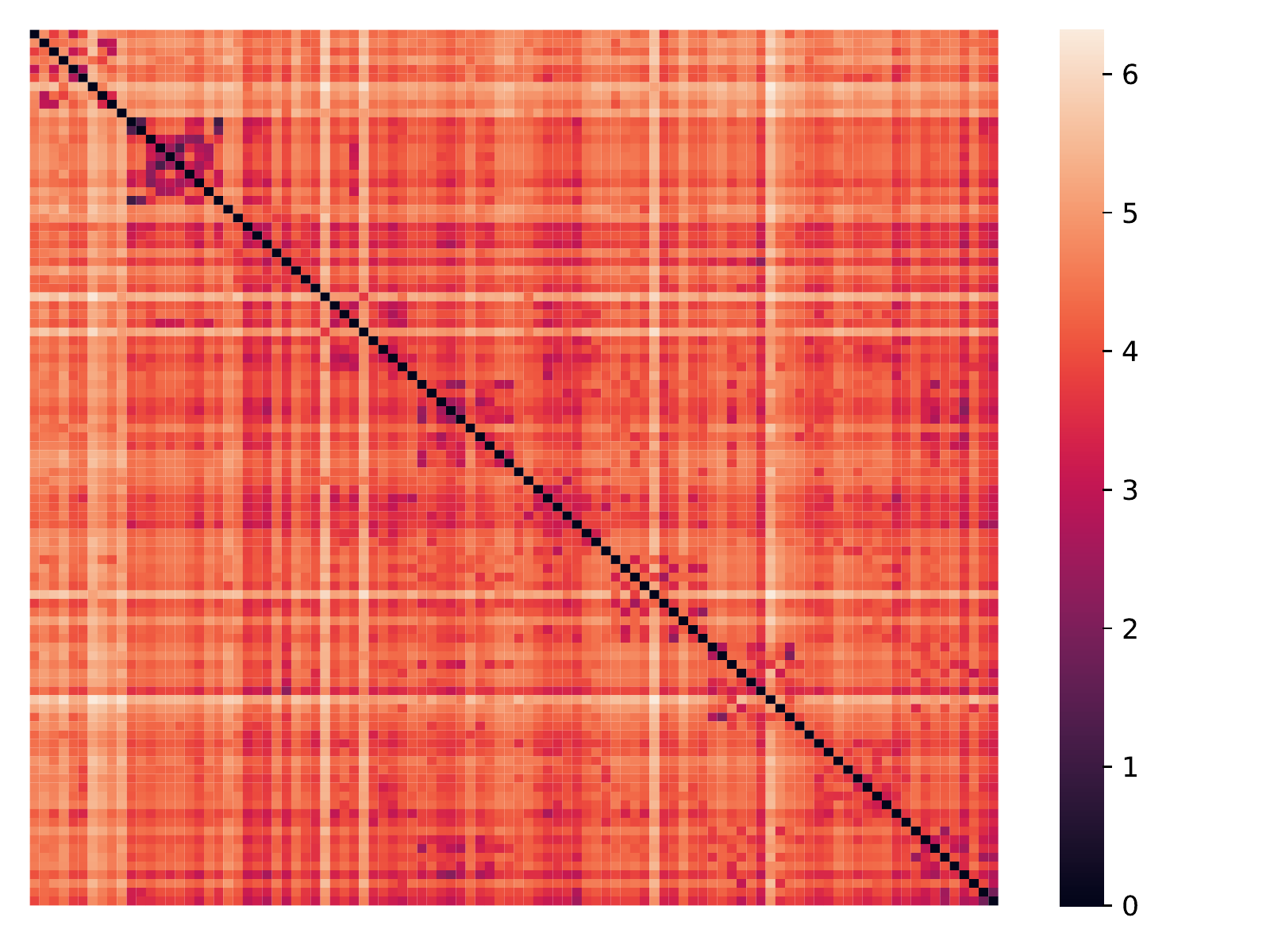}
	    \caption{Similarity matrix of MNIST images on the latent representation learned by a sparse autoencoder.}
	\end{subfigure}\hfill
	\begin{subfigure}[t]{0.32\textwidth}
	    \centering
        \includegraphics[width=0.9\textwidth]{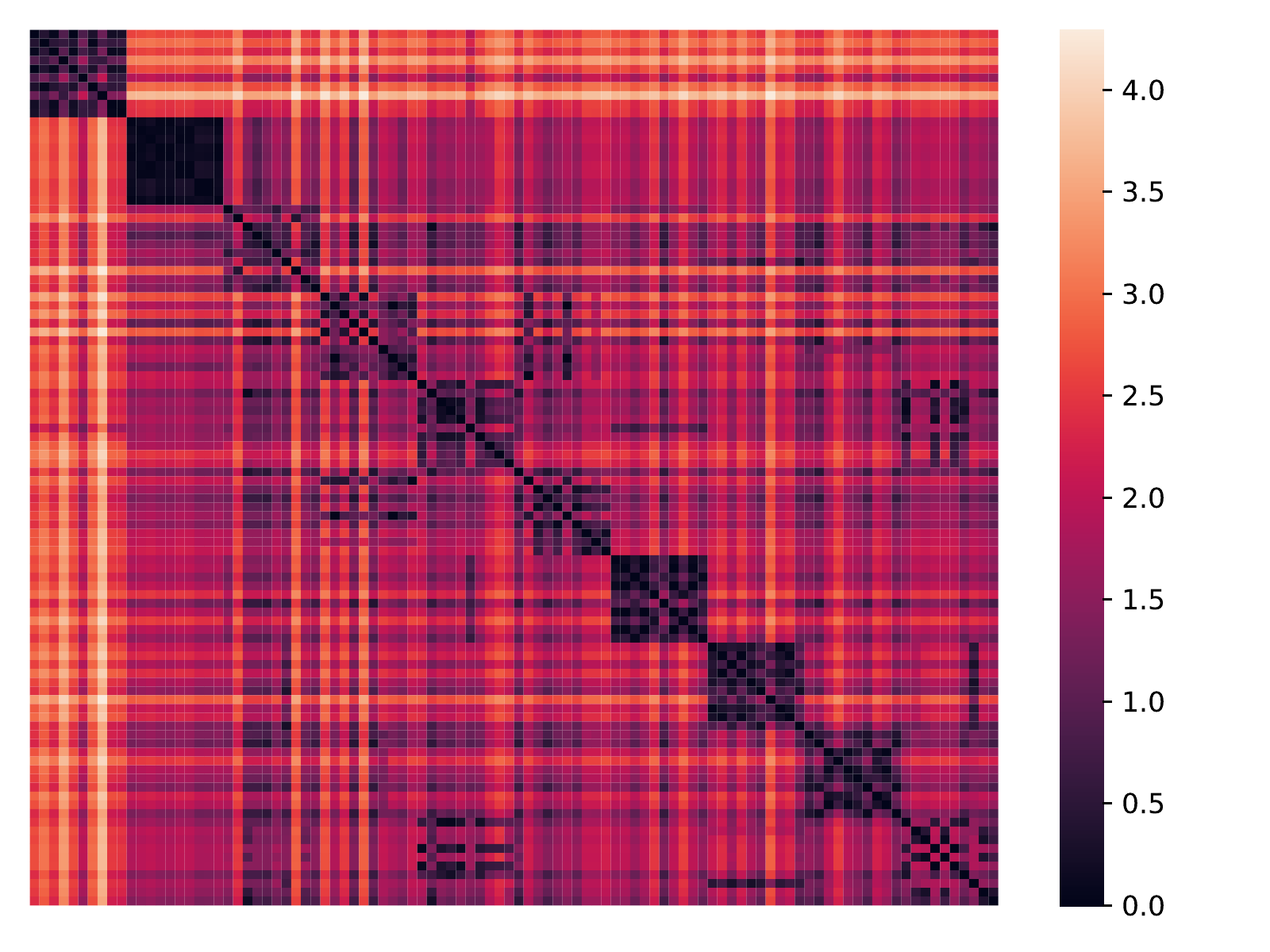}
	    \caption{Similarity matrix of MNIST images on the latent representation learned by a group-sparse autoencoder}
	\end{subfigure}
	\caption{Pairwise distances between the representations of MNIST test images for different latent representations.}
	\label{fig:distance_matrices}
\end{figure*}
\subsection{Real-data experiments}\label{sec:real_exp}

In this section, we demonstrate the clustering capability of our proposed group-sparse autoencoder on the MNIST Harwritten Digit dataset.

\textbf{Training.} We train our architecture on $60{,}000$ grayscale images of size $28 \times 28$ from the MNIST dataset. We set the number of groups to $\Gamma = 10$ each of size $s = 16$ and we unfold our encoder for $15$ iterations. We set the group-sparsity-inducing bias to $\lambda = 0.2$ and train the network for $300$ epochs with the Adam optimizer using a learning rate $\eta = 10^{-3}$. Using the same approach, settings, and a sparsity-inducing bias of $\lambda = 0.03$, we train a sparse autoencoder \citep{TDB18, NWH19} with $\bm{A} \in \mathcal{R}^{784 \times 160}$ as baseline.

\textbf{Results.} We observe that the weights learned by the group-sparse network (\cref{fig:mnist_dictionariesc}) are more interpretable than  those learned by a sparse network (\cref{fig:mnist_dictionariesa}). The sparse dictionaries of \cref{fig:mnist_dictionariesa} resemble the individual pen strokes of handwritten digits and do not reflect the underlying class membership. On the other hand, the group-sparse model is able to learn higher level features, shown in \cref{fig:mnist_dictionariesc}, that naturally resemble members of the cluster and reflect the ten distinct digit classes of the dataset.

\cref{fig:mnist_dictionariesb,fig:mnist_dictionariesd} shows examples of the sparse supports recovered using a sparse and a group-sparse architecture, respectively. We clearly observe that while the sparse support is spread throughout the coding vector, in stark contrast the group-sparse one is heavily structured. This meaningful structured connectivity (i.e., grouping of the dictionary atoms) is especially noteworthy considering that the network was trained \emph{without} the use of the ground truth class label information.

\begin{table}[t]
    \centering
    \begin{tabular}{cccc}
	\toprule
        & $k$-means & Spectral clustering\\
        \midrule
        Raw representation & $54.46$ & $63.04$\\
        Sparse codes & $54.14$ & $55.49$\\
        Sparse codes\textsubscript{$\bm{+}$} & $60.94$ & $56.98$\\
        Group-sparse codes & $\mathbf{74.5}$ & $\mathbf{79.4}$\\
        Group-sparse codes\textsubscript{$\bm{+}$} & $\mathbf{80.09}$ & $\mathbf{80.27}$\\
        \bottomrule
    \end{tabular}
    \caption{Test accuracy on the MNIST dataset ($N = 10{,}000$) for all $10$ classes. The subscript $(\cdot)_{\bm{+}}$ indicates that the coefficients were made nonnegative and scaled to sum to $1$.}
    \label{tab:clustering_accuracy}
\end{table}

\cref{fig:distance_matrices} shows pairwise distances between $100$ MNIST test images of each of the $10$ digit classes with respect to the pixel basis, the sparse dictionary of \cref{fig:mnist_dictionariesa}, and the group-sparse dictionary of \cref{fig:mnist_dictionariesc}. We report that the similarity structure of the group-sparse dictionary lends itself most readily to standard similarity-based clustering algorithms.
The reported performance gains attained by using group-sparsity in application domains such as clustering highlight the importance more structured notions of sparsity.

Finally, in \cref{tab:clustering_accuracy} we evaluate the efficacy of the dictionaries in \cref{fig:mnist_dictionariesa,fig:mnist_dictionariesc} when used as a pre-processing step for clustering. In particular, we performed clustering using $k$-means with $10$ centroids and spectral clustering on the similarity graph of $1{,}000$ nearest neighbors. The input to these algorithms is the representation of the digits in the pixel basis (first row), the latent representations learned with a sparse autoencoder (second and third rows), or the group-sparse autoencoder (fourth and fifth rows). We observed that performing clustering on the group-structured representations dramatically improved performance over baselines, and over a sparse approach. This further solidifies our findings that group-sparse architectures perform a natural grouping of the unlabeled data.

\section{Conclusions}
\label{sec:concl}
In this work we studied the gradient dynamics of a group-sparse, shallow autoencoder. Motivated by the connection between group-sparsity and cluster membership, we first introduced the group-sparse generative family and highlighted its expressivity. We then proceeded into proving, under mild conditions, that training the architecture with gradient descent will result into the convergence of the weight matrix to a neighborhood of the true weights. Finally, we provided numerical and real data experiments to support the developed theory.

\textbf{Acknowledgement}\\
This work is supported by the National Science Foundation under Cooperative Agreement PHY-2019786 (The NSF AI Institute for Artificial Intelligence and Fundamental Interactions, http://iaifi.org/).

\printbibliography
\newpage
\section*{Appendix A: Group-norm bounds}
In this section we will provide proofs for \cref{prop:lower,prop:upper}.
\begin{proposition*}[Group-norm lower bound]
The norm of the term $\bm{A}_g^T\bm{A}^{\ast}_g\bm{x}^{\ast}_g$ is lower-bounded by
\begin{equation}
	\lVert \bm{A}_g^T\bm{A}^{\ast}_g\bm{x}^{\ast}_g \rVert_2 \geq B_{\min}(1 - \delta).
\end{equation}
\end{proposition*}
\begin{proof}
	We have
	\begin{align}
		\begin{split}
		\lVert \bm{A}_g^T\bm{A}^{\ast}_g\bm{x}^{\ast}_g \rVert_2 &= \lVert (\bm{A}_g^T\bm{A}^{\ast}_g -\bm{I})\bm{x}^{\ast}_g + \bm{x}^{\ast}_g \rVert_2\\
		&\geq \lVert\bm{x}^{\ast}_g \rVert_2 - \lVert(\bm{A}_g^T\bm{A}^{\ast}_g -\bm{I})\bm{x}^{\ast}_g\rVert_2\\
		&\geq \lVert\bm{x}^{\ast}_g \rVert_2 - \lVert\bm{A}_g^T\bm{A}^{\ast}_g -\bm{I}\rVert_2\lVert\bm{x}^{\ast}_g\rVert_2\\
		&\geq (1 - \lVert\bm{A}_g^T\bm{A}^{\ast}_g -\bm{I}\rVert_2)\lVert\bm{x}^{\ast}_g\rVert_2.
		\end{split}
	\end{align}
	For $\lVert\bm{A}_g^T\bm{A}^{\ast}_g -\bm{I}\rVert_2$ it holds that
	\begin{align}
		\begin{split}
		\lVert\bm{A}_g^T\bm{A}^{\ast}_g -\bm{I}\rVert_2 &= \lVert\bm{A}_g^T(\bm{A}^{\ast}_g -\bm{A}_g)\rVert_2\\
		&\leq \lVert\bm{A}_g^T\rVert_2\lVert\bm{A}^{\ast}_g -\bm{A}_g\rVert_2\\
		&\leq \delta.
		\end{split}
	\end{align}
	Therefore, we finally get
	\begin{align}
		\begin{split}
		\lVert \bm{A}_g^T\bm{A}^{\ast}_g\bm{x}^{\ast}_g \rVert_2 &\geq (1 - \lVert\bm{A}_g^T\bm{A}^{\ast}_g -\bm{I}\rVert_2)\lVert\bm{x}^{\ast}_g\rVert_2\\
		&\geq (1-\delta)\lVert\bm{x}^{\ast}_g\rVert_2\\
		&\geq (1-\delta) B_{\min}.
		\end{split}
	\end{align}
\end{proof}

\begin{proposition*}[Cross-term upper bound]
The norm of the term $\sum_{h \in S}\bm{A}_v^T\bm{A}^{\ast}_h\bm{x}^{\ast}_h$ is upper-bounded by
\begin{equation}
	\lVert \sum_{h \in S}\bm{A}_v^T\bm{A}^{\ast}_h\bm{x}^{\ast}_h \rVert_2 \leq \gamma B_{\max}(\mu_B + \delta).
\end{equation}
\end{proposition*}
\begin{proof}
	We have that
	\begin{align}
		\begin{split}
		\lVert \sum_{h \in S}\bm{A}_v^T\bm{A}^{\ast}_h\bm{x}^{\ast}_h \rVert_2 &\leq \sum_{h \in S}\lVert\bm{A}_v^T\bm{A}^{\ast}_h\bm{x}^{\ast}_h \rVert_2\\
		&\leq \sum_{h \in S}\lVert(\bm{A}_v^T - \bm{A}^{\ast T}_v + \bm{A}^{\ast T}_v)\bm{A}^{\ast}_h\rVert_2\lVert\bm{x}^{\ast}_h \rVert_2\\
		&\leq \sum_{h \in S}\lVert\bm{A}_v^T - \bm{A}^{\ast T}_v\rVert_2\lVert\bm{A}^{\ast}_h\rVert_2\lVert\bm{x}^{\ast}_h \rVert_2 + \sum_{h \in S}\lVert\bm{A}^{\ast T}_v\bm{A}^{\ast}_h\rVert_2\lVert\bm{x}^{\ast}_h \rVert_2\\
		&\leq \sum_{h \in S}\delta B_{\max} + \sum_{h \in S}\mu_B B_{\max}\\
		&\leq \gamma B_{\max}(\mu_B + \delta).
		\end{split}
	\end{align}
\end{proof}

\section*{Appendix B: Gradient computation}
In this section we show that the gradient with respect to $\bm{A}_g$ of the loss function
\begin{equation}
	\mathcal{L}(\bm{A}) = \tfrac{1}{2}\lVert\bm{y} - \bm{A}\sigma_\lambda(\bm{A}^T\bm{y})\rVert_2^2,
\end{equation}
is given by
\begin{equation}
	\nabla_{\bm{A}_g}\mathcal{L}(\bm{A}) = -\left(\bm{y} - \bm{A}\sigma_{\lambda}(\bm{A}^T\bm{y})\right)\sigma_{\lambda}^T(\bm{A}^T_g\bm{y}) - \bm{y}\left(\bm{y} - \bm{A}\sigma_{\lambda}(\bm{A}^T\bm{y})\right)^T\bm{A}_g\operatorname{diag}(\sigma_{\lambda}'(\bm{A}^T_g\bm{y})).
\end{equation}
We will first compute the gradient of $\mathcal{L}(\bm{A})$ with respect to a column $\bm{a}_i$, and then combine the results into a matrix. We have
\begin{align}
	\begin{split}
	\nabla_{\bm{a}_i}\mathcal{L}(\bm{A}) &= \nabla_{\bm{a}_i} \tfrac{1}{2}(\bm{y} - \bm{A}\sigma_{\lambda}(\bm{A}^T\bm{y}))^T(\bm{y} - \bm{A}\sigma_{\lambda}(\bm{A}^T\bm{y}))\\
	&= \tfrac{1}{2} \nabla_{\bm{a}_i} (\bm{y} - \bm{A}\sigma_{\lambda}(\bm{A}^T\bm{y})) \cdot (\bm{y} - \bm{A}\sigma_{\lambda}(\bm{A}^T\bm{y})) + \tfrac{1}{2} \nabla_{\bm{a}_i} (\bm{y} - \bm{A}\sigma_{\lambda}(\bm{A}^T\bm{y})) \cdot (\bm{y} - \bm{A}\sigma_{\lambda}(\bm{A}^T\bm{y}))\\
	&= -\nabla_{\bm{a}_i} \bm{A}\sigma_{\lambda}(\bm{A}^T\bm{y}) \cdot (\bm{y} - \bm{A}\sigma_{\lambda}(\bm{A}^T\bm{y}))\\
	&= -\nabla_{\bm{a}_i} \sum_i \bm{a}_i\sigma_{\lambda}(\bm{a}_i^T\bm{y}) \cdot (\bm{y} - \bm{A}\sigma_{\lambda}(\bm{A}^T\bm{y}))\\
	&=  -\nabla_{\bm{a}_i} \bm{a}_i\sigma_{\lambda}(\bm{a}_i^T\bm{y}) \cdot (\bm{y} - \bm{A}\sigma_{\lambda}(\bm{A}^T\bm{y})).
	\end{split}
\end{align}
For the gradient $\nabla_{\bm{a}_i} \bm{a}_i\sigma_{\lambda}(\bm{a}_i^T\bm{y})$ we have that
\begin{align}
	\begin{split}
	\nabla_{\bm{a}_i} \bm{a}_i\sigma_{\lambda}(\bm{a}_i^T\bm{y}) &= \nabla_{\bm{a}_i} \begin{bmatrix}
		a_{1i}\sigma_{\lambda}(\bm{a}_{i}^T\bm{y})\\
		\vdots\\
		a_{ni}\sigma_{\lambda}(\bm{a}_{i}^T\bm{y})
\end{bmatrix} =
\begin{bmatrix}
	\nabla_{a_{1i}} a_{1i}\sigma_{\lambda}(\bm{a}_{i}^T\bm{y}) & \dots & \nabla_{a_{1i}} a_{ni}\sigma_{\lambda}(\bm{a}_{i}^T\bm{y})\\
	\vdots & \ddots & \vdots\\
	\nabla_{a_{ni}} a_{1i}\sigma_{\lambda}(\bm{a}_{i}^T\bm{y}) & \dots & \nabla_{a_{ni}} a_{ni}\sigma_{\lambda}(\bm{a}_{i}^T\bm{y})
\end{bmatrix}\\
&= \begin{bmatrix}
	\sigma_{\lambda}(\bm{a}_{i}^T\bm{y}) + a_{1i}\sigma'_{\lambda}(\bm{a}_{i}^T\bm{y})y_1 & a_{2i}\sigma'_{\lambda}(\bm{a}_{i}^T\bm{y})y_1 & \dots & a_{ni}\sigma'_{\lambda}(\bm{a}_{i}^T\bm{y})y_1\\
	a_{1i}\sigma'_{\lambda}(\bm{a}_{i}^T\bm{y})y_2 & \sigma_{\lambda}(\bm{a}_{i}^T\bm{y}) +  a_{2i}\sigma'_{\lambda}(\bm{a}_{i}^T\bm{y})y_2 & \dots & a_{ni}\sigma'_{\lambda}(\bm{a}_{i}^T\bm{y})y_2\\
	\vdots & \vdots & \ddots & \vdots\\
	a_{1i}\sigma'_{\lambda}(\bm{a}_{i}^T\bm{y})y_n & a_{2i}\sigma'_{\lambda}(\bm{a}_{i}^T\bm{y})y_n & \dots & \sigma_{\lambda}(\bm{a}_{i}^T\bm{y}) + a_{ni}\sigma'_{\lambda}(\bm{a}_{i}^T\bm{y})y_n
\end{bmatrix}\\
&=\sigma_{\lambda}(\bm{a}_{i}^T\bm{y}) \bm{I} + \begin{bmatrix}
	a_{1i}\sigma'_{\lambda}(\bm{a}_{i}^T\bm{y})y_1 & a_{2i}\sigma'_{\lambda}(\bm{a}_{i}^T\bm{y})y_1 & \dots & a_{ni}\sigma'_{\lambda}(\bm{a}_{i}^T\bm{y})y_1\\
	a_{1i}\sigma'_{\lambda}(\bm{a}_{i}^T\bm{y})y_2 & a_{2i}\sigma'_{\lambda}(\bm{a}_{i}^T\bm{y})y_2 & \dots & a_{ni}\sigma'_{\lambda}(\bm{a}_{i}^T\bm{y})y_2\\
	\vdots & \vdots & \ddots & \vdots\\
	a_{1i}\sigma'_{\lambda}(\bm{a}_{i}^T\bm{y})y_n & a_{2i}\sigma'_{\lambda}(\bm{a}_{i}^T\bm{y})y_n & \dots & a_{ni}\sigma'_{\lambda}(\bm{a}_{i}^T\bm{y})y_n
\end{bmatrix}\\
&=\sigma_{\lambda}(\bm{a}_{i}^T\bm{y}) \bm{I} + \sigma'_{\lambda}(\bm{a}_{i}^T\bm{y})\bm{y}\bm{a}_i^T
\end{split}
\end{align}
Therefore the gradient with respect to the column $\bm{a}_i$ becomes
\begin{align}
	\begin{split}
	\nabla_{\bm{a}_i}\mathcal{L}(\bm{A}) &= -(\sigma_{\lambda}(\bm{a}_{i}^T\bm{y}) \bm{I} + \sigma'_{\lambda}(\bm{a}_{i}^T\bm{y})\bm{y}\bm{a}_i^T) (\bm{y} - \bm{A}\sigma_{\lambda}(\bm{A}^T\bm{y}))\\
	&= -\underbrace{\sigma_{\lambda}(\bm{a}_{i}^T\bm{y})(\bm{y} - \bm{A}\sigma_{\lambda}(\bm{A}^T\bm{y}))}_{\text{decoder}} - \underbrace{\sigma'_{\lambda}(\bm{a}_{i}^T\bm{y})\bm{y}\bm{a}_i^T (\bm{y} - \bm{A}\sigma_{\lambda}(\bm{A}^T\bm{y}))}_{\text{encoder}}.
	\end{split}
\end{align}
In order to put everything in a single matrix we will deal with the terms corresponding to the encoder and the decoder separately. For the decoder, we have
\begin{align}
	\begin{split}
	\nabla_{\bm{a}_i}\mathcal{L}(\bm{A})_{\text{dec}} &= -\sigma_{\lambda}(\bm{a}_{i}^T\bm{y})(\bm{y} - \bm{A}\sigma_{\lambda}(\bm{A}^T\bm{y}))\\
	\Rightarrow \nabla_{\bm{A}_g}\mathcal{L}(\bm{A})_{\text{dec}} &= \begin{bmatrix}
		 -\sigma_{\lambda}(\bm{a}_{g_1}^T\bm{y})(\bm{y} - \bm{A}\sigma_{\lambda}(\bm{A}^T\bm{y})) & \dots & -\sigma_{\lambda}(\bm{a}_{g_d}^T\bm{y})(\bm{y} - \bm{A}\sigma_{\lambda}(\bm{A}^T\bm{y}))
	\end{bmatrix}\\
	&= -(\bm{y} - \bm{A}\sigma_{\lambda}(\bm{A}^T\bm{y}))\sigma_{\lambda}^T(\bm{A}_{g}^T\bm{y}),
	\end{split}
\end{align}
where $g_i$ denotes the $i$-th column of $\bm{A}_g$. For the encoder, noting that $\bm{a}_i^T (\bm{y} - \bm{A}\sigma_{\lambda}(\bm{A}^T\bm{y}))$ is a scalar, we have
\begin{align}
	\begin{split}
	\nabla_{\bm{a}_i}\mathcal{L}(\bm{A})_{\text{enc}} &= -\sigma'_{\lambda}(\bm{a}_{i}^T\bm{y})\bm{y}\bm{a}_i^T (\bm{y} - \bm{A}\sigma_{\lambda}(\bm{A}^T\bm{y}))\\
	&= -\sigma'_{\lambda}(\bm{a}_{i}^T\bm{y})\bm{y}(\bm{y} - \bm{A}\sigma_{\lambda}(\bm{A}^T\bm{y}))^T\bm{a}_i\\
	\Rightarrow \nabla_{\bm{A}_g}\mathcal{L}(\bm{A})_{\text{enc}} &= \begin{bmatrix}
		-\sigma'_{\lambda}(\bm{a}_{g_1}^T\bm{y})\bm{y}(\bm{y} - \bm{A}\sigma_{\lambda}(\bm{A}^T\bm{y}))^T\bm{a}_{g_1} & \dots & -\sigma'_{\lambda}(\bm{a}_{g_d}^T\bm{y})\bm{y}(\bm{y} - \bm{A}\sigma_{\lambda}(\bm{A}^T\bm{y}))^T\bm{a}_{g_d}
	\end{bmatrix}\\
	&= -\bm{y}(\bm{y} - \bm{A}\sigma_{\lambda}(\bm{A}^T\bm{y}))^T\bm{A}_{g}\operatorname{diag}(\sigma'_{\lambda}(\bm{A}_{g}^T\bm{y})).
	\end{split}
\end{align}
Finally, the gradient of $\mathcal{L}(\bm{A})$ with respect to $\bm{A}_g$ is given by
\begin{equation}
	\label{eq:gradient}
	\nabla_{\bm{A}_g}\mathcal{L}(\bm{A}) = -\left(\bm{y} - \bm{A}\sigma_{\lambda}(\bm{A}^T\bm{y})\right)\sigma_{\lambda}^T(\bm{A}^T_g\bm{y}) - \bm{y}\left(\bm{y} - \bm{A}\sigma_{\lambda}(\bm{A}^T\bm{y})\right)^T\bm{A}_g\operatorname{diag}(\sigma_{\lambda}'(\bm{A}^T_g\bm{y})).
\end{equation}

\section*{Appendix C: Convergence of gradient descent}
\subsection*{Expected gradient}
In this section we will prove the convergence of gradient descent under our assumptions, and also provide proofs for \cref{thrm:aligned,thrm:main_res}. Note that we can write
\begin{align}
	\begin{split}
	\sigma_{\lambda}(\bm{A}_g^T \bm{y}) &= \mathbbm{1}_{\bm{x}_g \neq 0} \left(1 - \frac{\lambda}{\lVert\bm{A}_g^T\bm{y}\rVert_2}\right)\bm{A}_g^T\bm{y},\\
	\sigma'_{\lambda}(\bm{A}_g^T \bm{y}) &= \mathbbm{1}_{\bm{x}_g \neq 0} \left(1 - \frac{\lambda}{\lVert\bm{A}_g^T\bm{y}\rVert_2}\right) \bm{1}.
	\end{split}
\end{align}
These substitutions will lead to an approximate gradient $\nabla_{\bm{A}_g}\tilde{\mathcal{L}}(\bm{A})$ that is a good approximation of $\nabla_{\bm{A}_g}\mathcal{L}(\bm{A})$ \citep{RMB+18}
\begin{align}
	\begin{split}
	\nabla_{\bm{A}_g} \widetilde{\mathcal{L}}(\bm{A}) &= -\mathbbm{1}_{\bm{x}_g \neq 0} \left(1 - \frac{\lambda}{\lVert\bm{A}_g^T\bm{y}\rVert_2}\right)(\bm{y} - \bm{A}\sigma_{\lambda}(\bm{A}^T\bm{y}))\bm{y}^T\bm{A}_g - \mathbbm{1}_{\bm{x}_g \neq 0} \left(1 - \frac{\lambda}{\lVert\bm{A}_g^T\bm{y}\rVert_2}\right)\bm{y}(\bm{y} - \bm{A}\sigma_{\lambda}(\bm{A}^T\bm{y}))^T\bm{A}_g\\
	&= -\mathbbm{1}_{\bm{x}_g \neq 0} \left(1 - \frac{\lambda}{\lVert\bm{A}_g^T\bm{y}\rVert_2}\right)[(\bm{y} - \bm{A}\sigma_{\lambda}(\bm{A}^T\bm{y}))\bm{y}^T\bm{A}_g + \bm{y}(\bm{y} - \bm{A}\sigma_{\lambda}(\bm{A}^T\bm{y}))^T\bm{A}_g]\\
	&= -\mathbbm{1}_{\bm{x}_g \neq 0} \left(1 - \frac{\lambda}{\lVert\bm{A}_g^T\bm{y}\rVert_2}\right)[(\bm{y} - \bm{A}\sigma_{\lambda}(\bm{A}^T\bm{y}))\bm{y}^T + \bm{y}(\bm{y} - \bm{A}\sigma_{\lambda}(\bm{A}^T\bm{y}))^T]\bm{A}_g.
	\end{split}
\end{align}
Let us define $\tau_g = \left(1 - \frac{\lambda}{\lVert\bm{A}_{g}^T\bm{y}\rVert_2}\right)$ for every $g \in S$. Then, we can define a vector $\bm{\tau}$ such that
\begin{equation}
	\bm{\tau} = \begin{bmatrix}
		\smash{\underbrace{\begin{matrix}\tau_1 & \ldots & \tau_1\end{matrix}}_{d}} & \tau_2  & \ldots & \smash{\underbrace{\begin{matrix}\tau_{\gamma} & \ldots & \tau_{\gamma}\end{matrix}}_{d}}
		\end{bmatrix}^T.
		\vspace{1em}
\end{equation}
We can then write $\sigma_{\lambda}(\bm{A}^T\bm{y}) = \operatorname{diag}(\bm{\tau})\bm{A}^T_{S}\bm{y}$, and then the approximate gradient $ \nabla_{\bm{A}_g}\widetilde{\mathcal{L}}(\bm{A})$ becomes
\begin{equation}
	\nabla_{\bm{A}_g}\widetilde{\mathcal{L}}(\bm{A}) = -\mathbbm{1}_{\bm{x}_g \neq 0} \tau_g [(\bm{I} - \bm{A}_{S}\operatorname{diag}(\boldsymbol{\tau})\bm{A}_{S}^T)\bm{y}\bm{y}^T + \bm{y}\bm{y}^T(\bm{I} - \bm{A}_{S}\operatorname{diag}(\boldsymbol{\tau})\bm{A}_{S}^T)^T]\bm{A}_g.
\end{equation}
We will now take the expectation of the gradient. We have
\begin{align}
	\begin{split}
	\bm{G}_g &= -\mathbbm{1}_{\bm{x}_g \neq 0} \tau_g [(\bm{I} - \bm{A}_{S}\operatorname{diag}(\boldsymbol{\tau})\bm{A}_{S}^T)\bm{y}\bm{y}^T + \bm{y}\bm{y}^T(\bm{I} - \bm{A}_{S}\operatorname{diag}(\boldsymbol{\tau})\bm{A}_{S}^T)^T]\bm{A}_g\\
	&= -\mathbbm{1}_{\bm{x}_g^{\ast} \neq 0} \tau_g [(\bm{I} - \bm{A}_{S}\operatorname{diag}(\boldsymbol{\tau})\bm{A}_{S}^T)\bm{y}\bm{y}^T + \bm{y}\bm{y}^T(\bm{I} - \bm{A}_{S}\operatorname{diag}(\boldsymbol{\tau})\bm{A}_{S}^T)^T]\bm{A}_g + \epsilon\\
	&= \bm{G}_g^{(1)} + \bm{G}_g^{(2)} + \bm{\epsilon},
	\end{split}
\end{align}
where the term $\bm{\epsilon}$ was introduced because we changed the indicator of $\mathbbm{1}_{\bm{x}_g \neq 0}$ to $\mathbbm{1}_{\bm{x}_g^{\ast} \neq 0}$. This was done as we already proved that the support is correctly recovered afterr we apply the proximal operator; nonetheless, this introduced an error term, that has, however, bounded norm \citep{NWH19}.

Because of the unknown support, we will use Adam's law to compute the expected gradient as $\bm{G}_g^{(i)} = \mathbb{E}[\bm{G}_{g\mid S}^{(i)}]$. Noting that $\bm{y}\bm{y}^T = \sum_{h,v \in S}\bm{A}_h^{\ast}\bm{x}_h^{\ast}\bm{x}_v^{\ast T}\bm{A}_v^{\ast T}$, we have
\begin{align}
	\begin{split}
	\bm{G}_{g\mid S}^{(1)} &= -\mathbb{E}[\mathbbm{1}_{x_g^\ast \neq 0} \tau_g (\bm{I} - \bm{A}_{S}\operatorname{diag}(\bm{\tau})\bm{A}_{S}^T)\sum_{h, v \in S}\bm{A}_h^{\ast}\bm{x}_h^\ast \bm{x}_v^{\ast T} \bm{A}_v^{\ast T}\bm{A}_{g} \mid S]\\
	&= -\tau_g (\bm{I} - \bm{A}_{S}\operatorname{diag}(\bm{\tau})\bm{A}_{S}^T)\sum_{h, v \in S}\bm{A}_h^{\ast}\mathbb{E}[\mathbbm{1}_{x_g^\ast \neq 0}\bm{x}_h^\ast \bm{x}_v^{\ast T} \mid S]\bm{A}_v^{\ast T}\bm{A}_{g}\\
	&= -\tau_g (\bm{I} - \bm{A}_{S}\operatorname{diag}(\bm{\tau})\bm{A}_{S}^T)\sum_{h \in S}\bm{A}_h^{\ast}\bm{A}_h^{\ast T}\bm{A}_{g}\\
	&= -\tau_g (\bm{I} - \bm{A}_{S}\operatorname{diag}(\bm{\tau})\bm{A}_{S}^T)\bm{A}_g^{\ast}\bm{A}_g^{\ast T}\bm{A}_{g} + \bm{P}_1,
	\end{split}
\end{align}
where $\bm{P}_1 = -\tau_g (\bm{I} - \bm{A}_{S}\operatorname{diag}(\bm{\tau})\bm{A}_{S}^T)\sum_{h \neq g\in S}\bm{A}_h^{\ast}\bm{A}_h^{\ast T}\bm{A}_{g}$. Similarly, we can show that
\begin{equation}
	\bm{G}_{g\mid S}^{(2)} = -\tau_g \bm{A}_g^{\ast}\bm{A}_g^{\ast T}(\bm{I} - \bm{A}_{S}\operatorname{diag}(\bm{\tau})\bm{A}_{S}^T)\bm{A}_{g} + \bm{P}_2,
\end{equation}
with $\bm{P}_2 = -\tau_g \sum_{h \neq g\in S}\bm{A}_h^{\ast}\bm{A}_h^{\ast T}(\bm{I} - \bm{A}_{S}\operatorname{diag}(\bm{\tau})\bm{A}_{S}^T)\bm{A}_{g}$. Then, letting $\bm{\beta} = \mathbb{E}[\bm{P}_1 + \bm{P}_2] + \bm{\epsilon}$, the expected gradient $\bm{G}_g$ becomes
\begin{align}
	\begin{split}
	\bm{G}_g &= \mathbb{E}[\bm{G}_{g\mid S}^{(1)} + \bm{G}_{g\mid S}^{(2)}] +\bm{\beta}\\
	&= -\tau_g \mathbb{E}[(\bm{I} - \sum_{h\in S} \tau_h \bm{A}_h \bm{A}_h^T)\bm{A}_g^{\ast}\bm{A}_g^{\ast T}\bm{A}_{g} + \bm{A}_g^{\ast}\bm{A}_g^{\ast T}(\bm{I} - \sum_{h\in S} \tau_h \bm{A}_h \bm{A}_h^T)\bm{A}_{g}] + \bm{\beta}\\
	&= -2\tau_g p_g\bm{A}_g^{\ast}\bm{A}_g^{\ast T}\bm{A}_{g} + \tau_g\mathbb{E}[\sum_{h\in S} \tau_h \bm{A}_h \bm{A}_h^T\bm{A}_g^{\ast}\bm{A}_g^{\ast T}\bm{A}_{g} + \bm{A}_g^{\ast}\bm{A}_g^{\ast T}\sum_{h\in S} \tau_h \bm{A}_h \bm{A}_h^T\bm{A}_{g}] + \bm{\beta}\\
	&= -2\tau_g p_g\bm{A}_g^{\ast}\bm{A}_g^{\ast T}\bm{A}_{g} + \tau_g^2 p_g \bm{A}_g \bm{A}_g^T\bm{A}_g^{\ast}\bm{A}_g^{\ast T}\bm{A}_{g} + \tau_g^2 p_g \bm{A}_g^{\ast}\bm{A}_g^{\ast T}\bm{A}_g\bm{A}_g^T\bm{A}_{g}\\
	&\qquad + \tau_g\mathbb{E}[\sum_{h\neq g\in S} \tau_h \bm{A}_h \bm{A}_h^T\bm{A}_g^{\ast}\bm{A}_g^{\ast T}\bm{A}_{g} + \bm{A}_g^{\ast}\bm{A}_g^{\ast T}\sum_{h\neq g\in S} \tau_h \bm{A}_h \bm{A}_h^T\bm{A}_{g}] + \bm{\beta}\\
	&= -2\tau_g p_g\bm{A}_g^{\ast}\bm{A}_g^{\ast T}\bm{A}_{g} + \tau_g^2 p_g \bm{A}_g \bm{A}_g^T\bm{A}_g^{\ast}\bm{A}_g^{\ast T}\bm{A}_{g} + \tau_g^2 p_g \bm{A}_g^{\ast}\bm{A}_g^{\ast T}\bm{A}_g\\
	&\qquad + \tau_g\sum_{h\neq g\in [\Gamma]} p_{gh}\tau_h (\bm{A}_h \bm{A}_h^T\bm{A}_g^{\ast}\bm{A}_g^{\ast T}\bm{A}_{g} + \bm{A}_g^{\ast}\bm{A}_g^{\ast T} \bm{A}_h \bm{A}_h^T\bm{A}_{g}) + \bm{\beta}\\
	&= -\tau_g(2 -\tau_g) p_g\bm{A}_g^{\ast}\bm{A}_g^{\ast T}\bm{A}_{g} + \tau_g^2 p_g \bm{A}_g \bm{A}_g^T\bm{A}_g^{\ast}\bm{A}_g^{\ast T}\bm{A}_{g} + \tilde{\bm{\beta}},
	\end{split}
\end{align}
where $\tilde{\bm{\beta}} = \bm{\beta} +  \tau_g\sum_{h\neq g\in [\Gamma]} p_{gh}\tau_h (\bm{A}_h \bm{A}_h^T\bm{A}_g^{\ast}\bm{A}_g^{\ast T}\bm{A}_{g} + \bm{A}_g^{\ast}\bm{A}_g^{\ast T} \bm{A}_h \bm{A}_h^T\bm{A}_{g})$.
Finally, to introduce the ``direction'' $\bm{A}_g^{\ast} - \bm{A}_g$ we can write
\begin{align}
	\begin{split}
		\bm{G}_g =& -\tau_g(2 -\tau_g) p_g\bm{A}_g^{\ast}\bm{A}_g^{\ast T}\bm{A}_{g} + \tau_g^2 p_g \bm{A}_g \bm{A}_g^T\bm{A}_g^{\ast}\bm{A}_g^{\ast T}\bm{A}_{g} + \tilde{\bm{\beta}}\\
		& -\tau_g (2-\tau_g) p_g\bm{A}_g\bm{A}_g^{\ast T}\bm{A}_{g} +\tau_g (2-\tau_g) p_g\bm{A}_g\bm{A}_g^{\ast T}\bm{A}_{g}\\
		=& \tau_g(2-\tau_g)p_g (\bm{A}_g - \bm{A}_g^{\ast})\bm{A}_g^{\ast T}\bm{A}_{g} + \tau_g p_g\bm{A}_g[\tau_g\bm{A}_g^T\bm{A}_g^{\ast} - (2-\tau_g)\bm{I}]\bm{A}_g^{\ast T}\bm{A}_{g} + \tilde{\bm{\beta}}.
	\end{split}
\end{align}
\subsection*{Proofs of Theorems 1 and 2}
\begin{theorem*}[Gradient direction]
The \emph{inner product} between the $i$-th columns of $\bm{G}_g$ and $\bm{A}_g - \bm{A}^{\ast}_g$ is lower-bounded by
\begin{align}
	\begin{split}
	2\langle \bm{g}_{gi}, \bm{a}_i - \bm{a}^{\ast}_i \rangle \geq& \tau_g(2-\tau_g)p_g\alpha_{i}\lVert\bm{a}_i - \bm{a}_i^{\ast}\rVert_2^2 + \frac{1}{\tau_g (2 - \tau_g) p_g \alpha_i}\lVert \bm{g}_{gi}\rVert_2^2 -\frac{1}{\tau_g (2 - \tau_g) p_g \alpha_i}\lVert \bm{v}_i\rVert_2^2\\ &-O((\mu_B + \delta)^2\tfrac{\gamma^5}{\Gamma^3}),
	\end{split}
\end{align}
where $\alpha_i = \bm{a}^{\ast T}_i\bm{a}_i$ and $\lVert\bm{v}_i\rVert_2$ satisfies
\begin{equation}
	\lVert\bm{v}_i\rVert_2 \leq \tau_g(2-\tau_g)p_g(\omega_i\sqrt{d^2 + 1} + \delta)\lVert\bm{A}_g - \bm{A}^{\ast}_g\rVert_F
\end{equation}
with $\omega_i = \max_{j\neq i\in g} \lvert\bm{a}^{\ast T}_j\bm{a}_i\rvert$.
\end{theorem*}
\begin{proof}
	We can write a single column of the gradient $\bm{G}_g$ as
	\begin{align}
		\begin{split}
			\bm{g}_{gi} =& \tau_g(2-\tau_g)p_g (\bm{A}_g - \bm{A}_g^{\ast})\bm{A}_g^{\ast T}\bm{a}_{i} + \tau_g p_g\bm{A}_g[\tau_g\bm{A}_g^T\bm{A}_g^{\ast} - (2-\tau_g)\bm{I}]\bm{A}_g^{\ast T}\bm{a}_{i} + \tilde{\bm{\beta}}_i\\
			=& \tau_g(2-\tau_g)p_g (\bm{a}_i - \bm{a}_i^{\ast})\bm{a}_i^{\ast T}\bm{a}_{i} + \tau_g(2-\tau_g)p_g \sum_{i\neq j \in g}(\bm{a}_j - \bm{a}_j^{\ast})\bm{a}_j^{\ast T}\bm{a}_{i}\\
			&+ \tau_g p_g\bm{A}_g[\tau_g\bm{A}_g^T\bm{A}_g^{\ast} - (2-\tau_g)\bm{I}]\bm{A}_g^{\ast T}\bm{a}_{i} + \tilde{\bm{\beta}}_i.
		\end{split}
	\end{align}
	Denote $\bm{v}_i = \tau_g(2-\tau_g)p_g \sum_{i\neq j \in g}(\bm{a}_j - \bm{a}_j^{\ast})\bm{a}_j^{\ast T}\bm{a}_{i} + \tau_g p_g\bm{A}_g[\tau_g\bm{A}_g^T\bm{A}_g^{\ast} - (2-\tau_g)\bm{I}]\bm{A}_g^{\ast T}\bm{a}_{i} + \tilde{\bm{\beta}}_i$ and let $\alpha_i = \bm{a}_i^{\ast T}\bm{a}_{i}$. Then we can write $\lVert\bm{g}_{gi}\rVert_2^2$ as
	\begin{align}
		\begin{split}
			\lVert\bm{g}_{gi}\rVert_2^2 = \tau_g^2(2-\tau_g)^2p_g^2 \alpha_i^2 \lVert\bm{a}_i - \bm{a}_i^{\ast}\rVert_2^2 + 2\tau_g(2-\tau_g)p_g \alpha_i (\bm{a}_i - \bm{a}_i^{\ast})^T\bm{v}_i + \lVert\bm{v}_{i}\rVert_2^2.
		\end{split}
	\end{align}
	However $\bm{v}_i = \bm{g}_{gi} -\tau_g(2-\tau_g)p_g \alpha_i (\bm{a}_i - \bm{a}_i^{\ast})$ and therefore
	\begin{align}
		\label{eq:inner_prod}
		\begin{split}
			&2\tau_g(2-\tau_g)p_g \alpha_i (\bm{a}_i - \bm{a}_i^{\ast})^T(\bm{g}_{gi} -\tau_g(2-\tau_g)p_g \alpha_i (\bm{a}_i - \bm{a}_i^{\ast})) = \lVert\bm{g}_{gi}\rVert_2^2 - \lVert\bm{v}_{i}\rVert_2^2 -\tau_g^2(2-\tau_g)^2p_g^2 \alpha_i^2 \lVert\bm{a}_i - \bm{a}_i^{\ast}\rVert_2^2\\
			&\qquad\qquad\Rightarrow 2(\bm{a}_i - \bm{a}_i^{\ast})^T\bm{g}_{gi} = \tau_g(2-\tau_g)p_g\alpha_i\lVert\bm{a}_i-\bm{a}_i^{\ast}\rVert_2^2 + \tfrac{1}{\tau_g(2-\tau_g)p_g\alpha_i}\lVert\bm{g}_{gi}\rVert_2^2 - \tfrac{1}{\tau_g(2-\tau_g)p_g\alpha_i}\lVert\bm{v}_{i}\rVert_2^2.
		\end{split}
	\end{align}
	We will bound the norm of $\bm{v}_i$. We have
	\begin{align}
		\begin{split}
			\lVert\bm{v}_{i}\rVert_2 \leq& \tau_g(2-\tau_g)p_g \lVert\sum_{i\neq j \in g}(\bm{a}_j - \bm{a}_j^{\ast})\bm{a}_j^{\ast T}\bm{a}_{i}\rVert_2 + \tau_g p_g\lVert\bm{A}_g[\tau_g\bm{A}_g^T\bm{A}_g^{\ast} - (2-\tau_g)\bm{I}]\bm{A}_g^{\ast T}\bm{a}_{i}\rVert_2 + \lVert\tilde{\bm{\beta}}_i\rVert_2\\
			\leq&  \tau_g(2-\tau_g)p_g \sum_{i\neq j \in g}\lVert\bm{a}_j - \bm{a}_j^{\ast}\rVert_2\lvert\bm{a}_j^{\ast T}\bm{a}_{i}\rvert + \tau_g p_g\lVert\bm{A}_g\rVert_2\lVert\tau_g\bm{A}_g^T\bm{A}_g^{\ast} - (2-\tau_g)\bm{I}\rVert_2\lVert\bm{A}_g^{\ast T}\rVert_2\lVert\bm{a}_{i}\rVert_2 + \lVert\tilde{\bm{\beta}}_i\rVert_2\\
			\leq&  \tau_g(2-\tau_g)p_g \sum_{i\neq j \in g}\lVert\bm{a}_j - \bm{a}_j^{\ast}\rVert_2\max_{j\neq i \in g}\lvert\bm{a}_j^{\ast T}\bm{a}_{i}\rvert + \tau_g p_g\lVert\tau_g\bm{A}_g^T\bm{A}_g^{\ast} - (2-\tau_g)\bm{I}\rVert_2 + \lVert\tilde{\bm{\beta}}_i\rVert_2,
		\end{split}
	\end{align}
	where we simplified the expression since $\lVert\bm{A}_g\rVert_2 = 1$ and $\lVert\bm{a}_i\rVert_2 = 1$ (both follow from the assumption that $\bm{A}_g^T\bm{A}_g = \bm{I}$). We also assume that $\lVert\bm{A}_g^{\ast}\rVert_2 = 1$. For the first term, we can see that
	\begin{align}
		\begin{split}
			(\sum_{j \neq i\in g}\lVert\bm{a}_j - \bm{a}_j^{\ast}\rVert_2)^2 \leq (\sum_{j\in g}\lVert\bm{a}_j - \bm{a}_j^{\ast}\rVert_2)^2 &= \sum_{j\in g}\lVert\bm{a}_j - \bm{a}_j^{\ast}\rVert_2^2 + \sum_{i, j\neq i \in g}\lVert\bm{a}_i - \bm{a}_i^{\ast}\rVert_2\lVert\bm{a}_j - \bm{a}_j^{\ast}\rVert_2\\
			&\leq \lVert\bm{A}_g - \bm{A}_g^{\ast}\rVert_F^2 + d^2 \max_{i\in g}\lVert\bm{a}_j - \bm{a}_j^{\ast}\rVert_2^2.
		\end{split}
	\end{align}
	However, $\max_{i\in g}\lVert\bm{a}_j - \bm{a}_j^{\ast}\rVert_2^2 \leq \sum_{i\in g}\lVert\bm{a}_j - \bm{a}_j^{\ast}\rVert_2^2 = \lVert\bm{A}_g - \bm{A}_g^{\ast}\rVert_F^2$, and therefore
	\begin{equation}
			(\sum_{j \neq i\in g}\lVert\bm{a}_j - \bm{a}_j^{\ast}\rVert_2)^2 \leq (d^2 + 1)\lVert\bm{A}_g - \bm{A}_g^{\ast}\rVert_F^2 \Rightarrow \sum_{j \neq i\in g}\lVert\bm{a}_j - \bm{a}_j^{\ast}\rVert_2 \leq \sqrt{d^2 + 1}\lVert\bm{A}_g - \bm{A}_g^{\ast}\rVert_F.
	\end{equation}
	Denote $\omega_i = \max_{j\neq i \in g} \lvert\bm{a}_j^{\ast T}\bm{a}_{i}\rvert$. We can temporarily rewrite the norm of $\bm{v}_i$ as
	\begin{align}
		\begin{split}
		\lVert\bm{v}_{i}\rVert_2 &\leq \tau_g(2-\tau_g)p_g \omega_i \sqrt{d^2 + 1}\lVert\bm{A}_g - \bm{A}_g^{\ast}\rVert_F + \tau_g p_g\lVert\tau_g\bm{A}_g^T\bm{A}_g^{\ast} - (2-\tau_g)\bm{I}\rVert_2 + \lVert\tilde{\bm{\beta}}_i\rVert_2\\
		&\leq \tau_g(2-\tau_g)p_g \omega_i \sqrt{d^2 + 1}\lVert\bm{A}_g - \bm{A}_g^{\ast}\rVert_F + \tau_g p_g\lVert\tau_g\bm{A}_g^T\bm{A}_g^{\ast} - (2-\tau_g)\bm{I}\rVert_F + \lVert\tilde{\bm{\beta}}_i\rVert_2,
		\end{split}
	\end{align}
	since $\lVert\cdot\rVert_2 \leq \lVert\cdot\rVert_F$. At this point, we want the second norm to be small (specifically on the order of $(2-\tau_g)\delta^2 = (2-\tau_g)\delta\lVert\bm{A}_g - \bm{A}_g^{\ast}\rVert_F$), so we can incorporate it into the first term of \eqref{eq:inner_prod}. We have
	\begin{align}
		\begin{split}
		\lVert\tau_g\bm{A}_g^T\bm{A}_g^{\ast} - (2-\tau_g)\bm{I}\rVert_F^2 &= (2-\tau_g)^2 \lVert\bm{I}\rVert_F^2 + \tau_g^2\lVert\bm{A}_g^T\bm{A}_g^{\ast}\rVert_F^2 - 2(2-
		\tau_g)\tau_g\langle\bm{I}, \bm{A}_g^T\bm{A}_g^{\ast}\rangle_F\\
		&= (2-\tau_g)^2 n + \tau_g^2\lVert\bm{A}_g^T\bm{A}_g^{\ast}\rVert_F^2 - 2(2-
		\tau_g)\tau_g\langle\bm{A}_g^T,\bm{A}_g^{\ast}\rangle_F,
		\end{split}
	\end{align}
	where $\langle\cdot,\cdot\rangle_F$ denotes the \emph{Frobenious inner product}. However, it holds that
	\begin{align}
		\begin{split}
		\lVert\bm{A}_g-\bm{A}_g^{\ast}\rVert_F \leq \delta \Rightarrow& \lVert\bm{A}_g-\bm{A}_g^{\ast}\rVert_F^2 \leq \delta^2\\
		\Rightarrow& \lVert\bm{A}_g\rVert_F^2 + \lVert\bm{A}_g^{\ast}\rVert_F^2 - 2\langle\bm{A}_g^T,\bm{A}_g^{\ast}\rangle_F \leq \delta^2\\
		\Rightarrow&\langle\bm{A}_g^T,\bm{A}_g^{\ast}\rangle_F \geq n -\tfrac{\delta^2}{2}.
		\end{split}
	\end{align}
	Therefore we have
	\begin{equation}
		\lVert\tau_g\bm{A}_g^T\bm{A}_g^{\ast} - (2-\tau_g)\bm{I}\rVert_F^2 \leq (2-\tau_g)^2 n + \tau_g^2\lVert\bm{A}_g^T\bm{A}_g^{\ast}\rVert_F^2 - 2(2-
		\tau_g)\tau_g(n -\tfrac{\delta^2}{2}).
	\end{equation}
	We can now find values of $\lVert\bm{A}_g^T\bm{A}_g^{\ast}\rVert_F^2$ that make the above quantity be below $(2-\tau_g)\delta^2$. We have
	\begin{align}
		\begin{split}
			&(2-\tau_g)^2 n + \tau_g^2\lVert\bm{A}_g^T\bm{A}_g^{\ast}\rVert_F^2 - 2(2-
			\tau_g)\tau_g(n -\tfrac{\delta^2}{2}) \leq (2-\tau_g)\delta^2\\
			\Rightarrow& \lVert\bm{A}_g^T\bm{A}_g^{\ast}\rVert_F^2 \leq\tfrac{1}{\tau_g^2}[3(2-\tau_g)(\tau_g-\tfrac{2}{3})n + (1-\tau_g)(2-\tau_g) \delta^2].
		\end{split}
	\end{align}
	Therefore, assuming that $\lVert\bm{A}_g^T\bm{A}_g^{\ast}\rVert_F^2 \leq\tfrac{1}{\tau_g^2}[3(2-\tau_g)(\tau_g-\tfrac{2}{3})n + (1-\tau_g)(2-\tau_g) \delta^2]$ we can rewrite $\lVert\bm{v}_i\rVert_2$ as
	\begin{align}
		\begin{split}
		\lVert\bm{v}_{i}\rVert_2 &\leq \tau_g(2-\tau_g)p_g \omega_i \sqrt{d^2 + 1}\lVert\bm{A}_g - \bm{A}_g^{\ast}\rVert_F + \tau_g p_g(2-\tau_g)\delta\lVert\bm{A}_g-\bm{A}_g^{\ast}\rVert_F + \lVert\tilde{\bm{\beta}}_i\rVert_2\\
		&\leq\tau_g(2-\tau_g)p_g(\omega_i \sqrt{d^2 +1} + \delta)\lVert\bm{A}_g-\bm{A}_g^{\ast}\rVert_F + \lVert\tilde{\bm{\beta}}_i\rVert_2.
		\end{split}
	\end{align}
	We will now deal with the term $\lVert\tilde{\bm{\beta}}_i\rVert_2$. Remember that
	\begin{equation}
		\tilde{\bm{\beta}}_i =  \tau_g\sum_{h\neq g\in [\Gamma]} p_{gh}\tau_h (\bm{A}_h \bm{A}_h^T\bm{A}_g^{\ast}\bm{A}_g^{\ast T}\bm{A}_{g} + \bm{A}_g^{\ast}\bm{A}_g^{\ast T} \bm{A}_h \bm{A}_h^T\bm{A}_{g}) + \mathbb{E}[\bm{P}_1 + \bm{P}_2] + \bm{\epsilon}.
	\end{equation}
	Working first with the term $\mathbb{E}[\bm{P}_1 + \bm{P}_2]$ we have that
	\begin{align}
		\begin{split}
			\lVert\mathbb{E}[\bm{P}_1]\rVert_2 &= -\tau_g \mathbb{E}[(\bm{I} - \bm{A}_{S}\operatorname{diag}(\bm{\tau})\bm{A}_{S}^T)\sum_{h \neq g\in S}\bm{A}_h^{\ast}\bm{A}_h^{\ast T}\bm{A}_{g}]\\
			&\leq \tau_g p_{gh}\lVert\sum_{h\neq g \in [\Gamma]}\bm{A}_h^{\ast}\bm{A}_h^{\ast T}\bm{A}_g\rVert_2 + \tau_g\lVert\sum_{h\neq g,v \in [\Gamma]}\tau_v p_{ghv}\bm{A}_v^{\ast}\bm{A}_v^{\ast T}\bm{A}_h\bm{A}_h^T\bm{A}_g\rVert_2.
		\end{split}
	\end{align}
	However, note the second term is at best $O(p_{gh}) = O(\tfrac{\gamma^3}{\Gamma^2})$ (when $v \in \{g, h\}$), and of smaller order otherwise. For the first term we have
	\begin{align}
		\begin{split}
			\tau_g p_{gh}\lVert\sum_{h\neq g \in [\Gamma]}\bm{A}_h^{\ast}\bm{A}_h^{\ast T}\bm{A}_g\rVert_2 &= \tau_g p_{gh}\lVert\sum_{h\neq g \in [\Gamma]}\bm{A}_h^{\ast}\bm{A}_h^{\ast T}(\bm{A}_g^{\ast} - \bm{A}_g^{\ast} + \bm{A}_g)\rVert_2\\
			&\leq \tau_g p_{gh}\sum_{h\neq g \in [\Gamma]}(\mu_B + \delta)\\
			&\leq \tau_g p_{gh}\gamma (\mu_B + \delta)\\
			&= O(p_{gh}\gamma (\mu_B + \delta))\\
			&= O(\tfrac{\gamma^3}{\Gamma^2}(\mu_B + \delta)).
		\end{split}
	\end{align}
	Therefore, $\mathbb{E}[\bm{P}_1] = O(\tfrac{\gamma^3}{\Gamma^2}(\mu_B + \delta))$. Similarly, we can show that $\mathbb{E}[\bm{P}_2] = O(\tfrac{\gamma^3}{\Gamma^2}(\mu_B + \delta))$. Moreover, we know that $\lVert\bm{\epsilon}\rVert_2 = O(n^{-w(1)})$ \citep{NWH19}. We can achieve the exact bound for $\lVert\bm{\beta}\rVert_2$, and then, assuming $O(n^{-w(1)})$ is dominated by $O(\tfrac{\gamma^3}{\Gamma^2}(\mu_B + \delta))$, we finally have
	\begin{equation}
		\lVert\bm{v}_i\rVert_2 \leq \tau_g(2-\tau_g)p_g(\omega_i \sqrt{d^2 +1} + \delta)\lVert\bm{A}_g-\bm{A}_g^{\ast}\rVert_F + O(\tfrac{\gamma^3}{\Gamma^2}(\mu_B + \delta)).
	\end{equation}
	Then the squared norm of $\bm{v}_i$ satisfies
	\begin{equation}
		\label{eq:v_norm}
		\lVert\bm{v}_i\rVert_2^2 \leq 2\tau_g^2(2-\tau_g)^2p_g^2(\omega_i \sqrt{d^2 +1} + \delta)^2\lVert\bm{A}_g-\bm{A}_g^{\ast}\rVert_F^2 + O(\tfrac{\gamma^6}{\Gamma^4}(\mu_B + \delta)^2),
	\end{equation}
	since $(a + b)^2 <= 2a^2 + 2b^2$.
	By substituting \eqref{eq:v_norm} in \eqref{eq:inner_prod}, the proof is complete.
\end{proof}
\begin{theorem*}[Convergence to a neighborhood]
	Suppose that the learning rate $\eta$ is upper bounded by $\frac{1}{\tau_g(2-\tau_g)p_g\alpha_{\max}}$, the norm $\lVert\bm{A}_g^T\bm{A}_g^{\ast}\rVert_F^2$ is upper bounded by $\frac{1}{\tau_g^2}[3(2-\tau_g)(\tau_g -\frac{2}{3})n + (1-\tau_g)(2 -\tau_g)\delta^2]$, and that $\bm{G}_g$ is ``aligned'' with $\bm{A}_g -\bm{A}_g^{\ast}$, i.e.
	\begin{equation*}
		2\langle \bm{g}_{gi}, \bm{a}_i - \bm{a}^{\ast}_i \rangle \geq \kappa \lVert\bm{a}_i - \bm{a}^{\ast}_i\rVert_2^2 + \nu \lVert \bm{g}_{gi}\rVert_2^2 - \xi \lVert \bm{v}_i\rVert_2^2 - \varepsilon,
	\end{equation*}
	where $\kappa, \nu, \xi,$ and $\varepsilon$ are given by \cref{thrm:aligned}. Then it follows that
	\begin{equation*}
		\lVert\bm{A}^{k+1}_g - \bm{A}^{\ast}_g\rVert_F^2 \leq (1-\rho)\lVert\bm{A}^{k}_g - \bm{A}^{\ast}_g\rVert_F^2 + \eta d\varepsilon,
	\end{equation*}
	where $\rho = \eta\tau_g(2-\tau_g)p_g(\alpha_{\min}-\tfrac{2d(\omega_{\max}\sqrt{d^2 + 1} +\delta)^2}{\alpha_{\min}})$. If we further consider the assumptions for the recovery of the support, then
	\begin{equation*}
		\lVert\bm{A}^{k+1}_g - \bm{A}^{\ast}_g\rVert_F^2 \leq (1-\rho)\lVert\bm{A}^{k}_g - \bm{A}^{\ast}_g\rVert_F^2 + O(d\tfrac{\log^2n}{\Gamma^2}),
	\end{equation*}
	where $\omega_{\max} = \max_{i} \omega_i$, $\alpha_{\min} = \min_i \alpha_i$, and $\alpha_{\max} = \max_i \alpha_i$.
\end{theorem*}
\begin{proof}
	The gradient updates are done as follows
	\begin{equation}
		\bm{A}_g^{k+1} = \bm{A}_g^{k} - \eta \nabla_{\bm{A}_g^k}\mathcal{L}(\bm{A}^{k}).
	\end{equation}
	The, considering the norm between the true weights and the weights at iteration $k+1$, and denoting $\alpha_{\min} = \min_{i\in g} \bm{a}_i^{\ast T}\bm{a}_i, \alpha_{\max} = \max_{i\in g} \bm{a}_i^{\ast T}\bm{a}_i$, we have
	\begin{align}
		\begin{split}
			\lVert\bm{A}_g^{k+1} - \bm{A}_g^{\ast}\rVert_F^2 =& \sum_{i\in g}\lVert\bm{a}_i^{k+1} - \bm{a}_i^{\ast}\rVert_2^2 = \sum_{i\in g}\left[\lVert\bm{a}_i^{k} - \bm{a}_i^{\ast}\rVert_2^2 - 2\eta (\bm{a}_i^k - \bm{a}_i^{\ast})^T\bm{g}_{gi} + \eta^2\lVert\bm{g}_{gi}\rVert_2\right]\\
			\leq& \lVert\bm{A}_g^{k} - \bm{A}_g^{\ast}\rVert_F^2 + \eta^2\lVert\bm{g}_g\rVert_F\\
			&-\eta\sum_{i\in g}\left[\tau_g(2-\tau_g)p_g \alpha_i \lVert\bm{a}_i^{k} -\bm{a}_i^{\ast}\rVert_2^2 + \frac{1}{\tau_g (2 - \tau_g) p_g \alpha_i}\lVert \bm{g}_{gi}\rVert_2^2 -\frac{1}{\tau_g (2 - \tau_g) p_g \alpha_i}\lVert \bm{v}_i\rVert_2^2\right]\\
			\leq&\lVert\bm{A}_g^{k} - \bm{A}_g^{\ast}\rVert_F^2 + \eta^2\lVert\bm{g}_g\rVert_F\\
			&-\eta\left[\tau_g(2-\tau_g)p_g \alpha_{\min} \lVert\bm{A}_g^{k} -\bm{A}_g^{\ast}\rVert_F^2 + \frac{1}{\tau_g (2 - \tau_g) p_g \alpha_{\max}}\lVert \bm{G}_{g}\rVert_F^2 -\frac{1}{\tau_g (2 - \tau_g) p_g \alpha_{\min}}\lVert \bm{v}\rVert_F^2\right]\\
			\leq& [1-\eta\tau_g(2-\tau_g)p_g(\alpha_{\min} - \tfrac{2d(\omega_{\max}\sqrt{d^2 +1} +\delta)^2}{\alpha_{\min}})]\lVert\bm{A}_g^{k} - \bm{A}_g^{\ast}\rVert_F^2\\
			& -\eta(\frac{1}{\tau_g(2-\tau_g)p_g\alpha_{\max}} -\eta)\lVert\bm{G}_g\rVert_F^2 + O(\eta d(\mu_B + \delta)^2\frac{\gamma^5}{\Gamma^3})\\
			\leq& [1-\eta\tau_g(2-\tau_g)p_g(\alpha_{\min} - \tfrac{2d(\omega_{\max}\sqrt{d^2 +1} +\delta)^2}{\alpha_{\min}})]\lVert\bm{A}_g^{k} - \bm{A}_g^{\ast}\rVert_F^2 + O(\tfrac{1}{p_g}d(\mu_B + \delta)^2\frac{\gamma^5}{\Gamma^3})\\
			\leq& (1-\rho)\lVert\bm{A}_g^{k} - \bm{A}_g^{\ast}\rVert_F^2 + O(d(\mu_B + \delta)^2\frac{\gamma^4}{\Gamma^2}),
		\end{split}
	\end{align}
	where $\eta \leq \frac{1}{\tau_g(2-\tau_g)p_g\alpha_{\max}}$ and $\rho = \eta\tau_g(2-\tau_g)p_g(\alpha_{\min} - \tfrac{2d(\omega_{\max}\sqrt{d^2 +1} +\delta)^2}{\alpha_{\min}})$ Finally, \textbf{Lemma 1} states that the support is correctly recovered if $\gamma \leq \log n$ and $\max(\mu_B, \delta) = O(\frac{1}{\log n})$. Then, the error term becomes
	\begin{align}
		\begin{split}
			O(d(\mu_B + \delta)^2\frac{\gamma^4}{\Gamma^2}) = O(d \frac{\log^2n}{\Gamma^2}),
		\end{split}
	\end{align}
	which concludes the proof.
\end{proof}
\end{document}